\documentclass{eitnlp}
\usepackage{amsthm,array,tikz,pgfplots,algorithm}
\usepackage{adjustbox,colortbl}
\usepackage{placeins}
\usepackage[noend]{algpseudocode}
\usetikzlibrary{arrows.meta,positioning,shapes.geometric}
\usepgfplotslibrary{groupplots}
\pgfplotsset{compat=1.18}
\hypersetup{
  pdftitle={Nucleus Speculative Decoding: Plausibility-Aware Verification Beyond Exact Distribution},
  pdfauthor={Shuhao Li, Fanghua Ye, Wanyu Lin, Tianyu Yuan, Xiaoyu Shen}
}
\definecolor{sdblue}{HTML}{286B99}
\definecolor{nsdorange}{HTML}{D96D25}

\newtheorem{theorem}{Theorem}[section]
\newtheorem{corollary}[theorem]{Corollary}
\theoremstyle{remark}

\theoremstyle{plain}
\newcommand{\Vocab}{\mathcal{V}}

\newcommand{\TV}{D_{\mathrm{TV}}}

\newcommand{\MTPQualityDelta}{-10.63}

\newcommand{\MTPQualityDeltaAbs}{\expandafter\RemoveSign\MTPQualityDelta}
\def\RemoveSign-#1{#1}
\title{Nucleus Speculative Decoding: Plausibility-Aware Verification Beyond Exact Distribution}
\runningtitle{Nucleus Speculative Decoding}
\author{%
  Shuhao Li\textsuperscript{1,2}\quad
  Fanghua Ye\textsuperscript{3}\quad
  Wanyu Lin\textsuperscript{2}\quad
  Tianyu Yuan\textsuperscript{1\textdagger}\quad
  Xiaoyu Shen\textsuperscript{1\textdagger}}
\affiliation{%
  \textsuperscript{1}EIT-NLP Lab, Eastern Institute of Technology, Ningbo\\
  \textsuperscript{2}The Hong Kong Polytechnic University \quad \textsuperscript{3}Independent Researcher}
\authornotes{\textsuperscript{\textdagger}Corresponding authors.}
\contact{%
  \texttt{shuhao.li@connect.polyu.hk}\quad
  \texttt{tyyuan@eitech.edu.cn}\quad
  \texttt{xyshen@eitech.edu.cn}}
\begin{document}
\raggedbottom
\maketitle
\begin{abstract}

Speculative decoding (SD) accelerates autoregressive generation by letting a lightweight draft model propose tokens that a target model verifies in parallel.
However, its standard acceptance rule, designed for exact distribution correction, rejects tokens to which the draft assigns excess probability, even when the target finds them highly plausible: over half of the rejected tokens lie inside the target's own nucleus.
Since verification dominates the cost of SD, these rejections waste its most expensive operation.
We introduce Nucleus Speculative Decoding (NSD), a plausibility-aware verification rule that accepts a draft token if it passes the standard test or belongs to the target nucleus, while keeping SD's residual correction and bonus sampling unchanged.
NSD requires no training or extra forward passes and is compatible with any drafting mechanism.
Because it preserves SD's recovery rule, its deviation from the target distribution can be characterized exactly: we prove that the single-step total variation error equals the draft's excess probability inside the target nucleus, which is precisely NSD's gain in acceptance probability.
NSD thus trades fidelity for acceptance at an exact one-to-one rate, and we further derive sequence-level fidelity bounds.
Across six target models and five mathematical reasoning and code generation benchmarks, NSD achieves the highest throughput among all compared methods, with speedups of up to $5.16\times$ over autoregressive decoding and $3.15\times$ over standard SD, while improving accuracy over SD on five of the six models and nearly halving verification time per output token.
Our code is available at \url{https://github.com/EIT-NLP/Nucleus-Speculative-Decoding}.
\end{abstract}

\section{Introduction}
Large language models (LLMs) have demonstrated strong capabilities in language understanding, reasoning, and code generation \citep{chen2021,su2022welm,guo2025r1,yang2025qwen3,deepseekai2026deepseekv4}, but autoregressive (AR) decoding remains inherently sequential and computationally expensive \citep{miao2023specinfer,zhang2024draftverify}.
Speculative decoding (SD) alleviates this bottleneck by letting a lightweight draft model propose a block of tokens that the target model verifies in parallel \citep{leviathan2023,chen2023accelerating}.
Despite substantial progress in draft models \citep{chen2026dflash,li2024eagle,gloeckle2024mtp,li2025eagle3,cheng2026dspark}, target verification still dominates the cost of SD \citep{liu2026speculativedecodingperformanceillusion}, accounting for \(57.1\%\)--\(78.1\%\) of wall-clock time (Figure~\ref{fig:systems_and_rescue_opportunity}(a)).
Moreover, verification stops at the first rejection and discards the rest of the block \citep{leviathan2023,stern2018blockwise}, so the speedup hinges on how many tokens each verification pass can commit.

\begin{figure}[t]
\centering
\includegraphics[width=\linewidth]{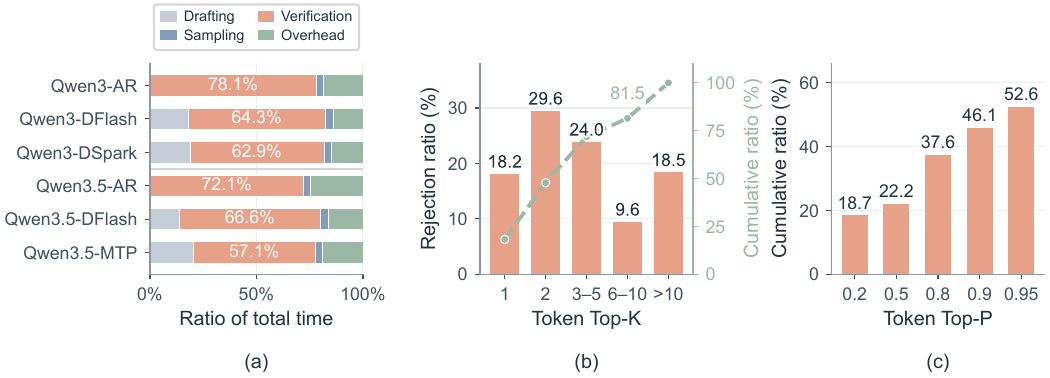}
\caption{Execution time breakdown and rejection analysis of speculative decoding.
(a) Execution time breakdown. Ratio of total time spent on drafting, verification, sampling, and overhead.
(b) Target Top-K rank from rejected tokens.
(c) Target Top-P distribution from rejected tokens.
}
\label{fig:systems_and_rescue_opportunity}
\end{figure}

Many of these rejections, however, are unnecessary.
SD accepts a draft token \(x \sim q\) with probability \(\min\{1, p(x)/q(x)\}\), where \(p\) is the target distribution \citep{chen2023accelerating}.
This rule is designed for exact distribution correction: it rejects \(x\) whenever the draft over-allocates probability to it (\(q(x) > p(x)\)), regardless of how plausible \(x\) is under the target.
As shown in Figure~\ref{fig:systems_and_rescue_opportunity}(b) and (c), \(81.5\%\) of rejected tokens rank within the target's top ten, and \(52.6\%\) lie inside the target nucleus at \(0.95\) \citep{holtzman2020}, i.e., among the tokens the target itself could produce under nucleus sampling.
SD thus spends its most expensive operation \emph{discarding tokens that the target considers acceptable}.

This observation motivates relaxing verification at the cost of exact distribution matching.
Recent relaxed verifiers pursue this goal by redesigning the verifier itself, introducing learned judges, logit-margin rules, regret budgets, or constrained acceptance distributions \citep{bachmann2025judge,song2026,feng2026asd,hao2026}.
In doing so, they either require additional training or replace SD's stochastic residual correction with a new recovery mechanism, so their deviation from the target distribution is controlled only indirectly rather than characterized exactly.
This raises a natural question: can we accept plausible proposals while leaving SD's sampling machinery untouched, and quantify exactly what this relaxation costs in fidelity?

We answer this question with \textbf{Nucleus Speculative Decoding} (NSD), a plausibility-aware verification rule that accepts a draft token if it passes the standard ratio test \emph{or} lies in the target nucleus \(S_\theta\), the smallest set of highest-probability target tokens whose cumulative probability reaches a threshold \(\theta\).
All other components of SD, including residual correction and bonus sampling, \emph{remain unchanged}.
Since the nucleus is computed from target probabilities already produced during verification, NSD requires no training or extra forward pass, and it composes with any drafting paradigm \citep{li2024eagle,li2024eagle2,li2025eagle3,gloeckle2024mtp,deepseekv3,chen2026dflash,cheng2026dspark,fu2026nemotron}.

Because NSD keeps SD's recovery rule, its output distribution admits a closed form, so the cost of the relaxation can be characterized exactly rather than merely bounded.
We prove that the single-step total variation distance \citep{gibbs2002choosingboundingprobabilitymetrics} between NSD's output distribution and the target distribution equals \(E_S = \sum_{v \in S_\theta} [q(v) - p(v)]_+\), the draft's excess probability inside the target nucleus, which is also exactly NSD's gain in acceptance probability over SD.
NSD thus trades fidelity for acceptance at an \emph{exact one-to-one rate}: the deviation never exceeds SD's own rejection probability, vanishes whenever \(q \le p\) on \(S_\theta\), and grows monotonically with \(\theta\).
We further derive a sequence-level fidelity bound via maximal coupling \citep{d3b61654b1204385b1a2e31f324d0229}.

Across six Nemotron and Qwen3 target models and five mathematical reasoning and code generation benchmarks, NSD achieves the highest throughput among all compared methods in every comparison, with speedups of up to \(5.16\times\) over AR and \(3.15\times\) over SD (\(4.23\times\) and \(1.85\times\) on average), while improving overall accuracy over SD on five of the six models (\(+0.40\) points on average).
Its benefits also transfer to DFlash, EAGLE-3, and multi-token prediction drafters.

These gains stem directly from the mechanism motivated above.
The tokens rescued by NSD are highly plausible under the target: \(94.1\%\) rank within its top five, and \(86.7\%\) lie in the core rather than at the boundary of the nucleus.
By retaining them, NSD reduces verification calls per output token by \(47.8\%\) at nearly unchanged per-call latency, cutting verification time per output token by \(48.3\%\) relative to SD and by \(17.3\%\)--\(20.9\%\) relative to other relaxed verifiers.
Together, these results show that accepting plausible tokens, rather than only distribution-corrected ones, amortizes the dominant cost of SD over substantially more tokens, offering a simple and theoretically grounded path toward faster speculative decoding.

\section{Related Work}
\label{sec:related_work}

\paragraph{Speculative Decoding.}
Early blockwise parallel decoding predicts several positions together and
retains the longest prefix validated by a scoring model \citep{stern2018blockwise}.
Speculative decoding accelerates autoregressive generation by using a
lightweight draft model to propose future tokens that will be verified in
parallel by the target model
\citep{leviathan2023,chen2023accelerating}.
For a draft proposal sampled from distribution $q$, speculative decoding accepts tokens with probability $\min\{1,p(x)/q(x)\}$ and applies residual correction upon rejection to preserve the target distribution
$p$.
Its efficiency depends on the number of draft tokens retained from
verification.
SpecInfer organizes proposals into token trees for parallel target model
verification \citep{miao2023specinfer}, while Draft \& Verify generates
proposals by skipping target model layers \citep{zhang2024draftverify}.

Recent advances have improved speculative decoding through stronger draft models, including feature-based drafters
\citep{li2024eagle,li2024eagle2,li2025eagle3}, multi-token prediction
modules \citep{gloeckle2024mtp,deepseekv3}, parallel diffusion drafting
\citep{chen2026dflash}, semi-autoregressive proposal strategies
\citep{cheng2026dspark}, and unified autoregressive diffusion models
\citep{fu2026nemotron}.
These approaches primarily improve the quality or efficiency of draft
generation. In contrast, NSD modifies the verification procedure and is complementary to existing drafting mechanisms. Therefore, NSD can
be applied across different speculative decoding paradigms.

\paragraph{Relaxed Verification.}
Several works relax exact verification to improve acceptance
rates under an accuracy-efficiency trade-off.
The lenience extension of speculative decoding modifies the likelihood
ratio criterion, while typical acceptance and related methods introduce
target probability or ranking thresholds
\citep{leviathan2023,cai2024}.
Other approaches explore adaptive verification rules, approximate sampling, learned judges, or regret constraints
\citep{song2026,urban2026adaptive,hao2026,bachmann2025,holsman2025,feng2026asd}.
These methods generally rely on learned components or optimization objectives.
Training-free verification based on draft-target KL divergence has also
been explored as an alternative to learned judges \citep{sun2026}. Inspired by nucleus sampling \citep{holtzman2020}, NSD introduces a
different relaxation: a draft token is accepted if it satisfies the
standard acceptance rule or belongs to the target nucleus.
Unlike approaches that replace the verification criterion, NSD retains the
original residual correction and bonus sampling mechanisms, providing a
simple nucleus-based acceptance condition using target probabilities
already computed during verification.

\section{Nucleus Speculative Decoding}
\label{sec:method}

We introduce Nucleus Speculative Decoding (NSD), a simple yet effective method for speculative decoding.
NSD accepts a draft proposal if it passes the standard 
acceptance verification or belongs to the target model's top-p set. This method can extend the acceptance length using target
probabilities already computed during verification. Figure~\ref{fig:nsd_verification_overview} illustrates the workflow of our method.
\begin{figure}[t]
\centering
\includegraphics[width=\linewidth]{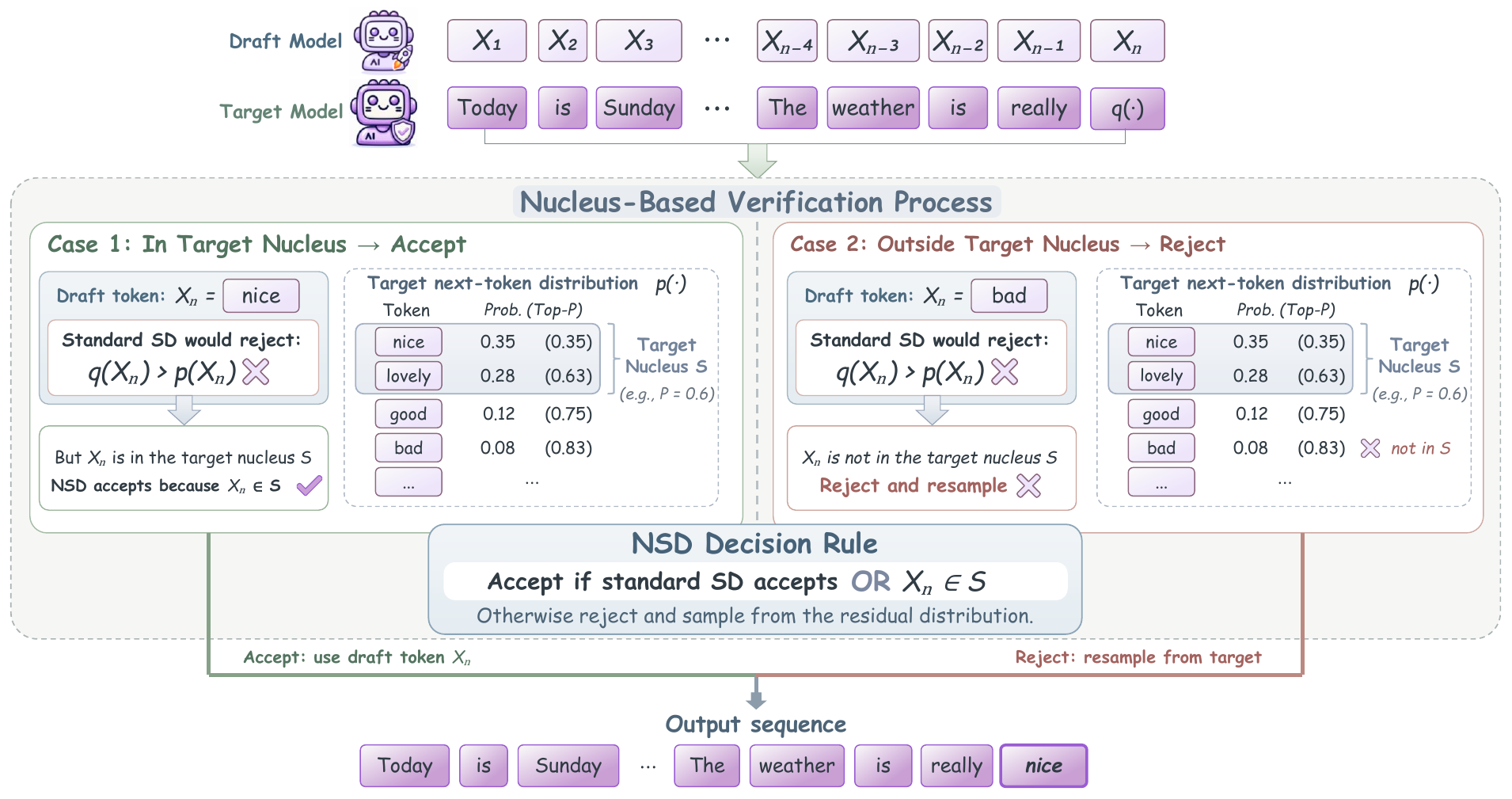}
\caption{Overview of Nucleus Speculative Decoding (NSD).}
\label{fig:nsd_verification_overview}
\end{figure}

\paragraph{Background on Speculative Decoding}
\label{sec:nsd_setup}

Let
$p(\cdot|h)$ and $q(\cdot|h)$ denote the target and draft
distributions conditioned on a prefix $h$. Given a committed prefix
$h_0$, the draft model generates a block of $K$ candidate tokens,
and the target model evaluates their conditional probabilities in
parallel. To preserve the exact target distribution, standard speculative
decoding accepts each draft token $x_i$ using the stochastic rule \(a_i^{\mathrm{SD}}(x_i)=\min\{1,p_i(x_i)/q_i(x_i)\}\) where $p_i$ and $q_i$ denote the target and draft distributions at
position $i$. Verification continues until the first rejection, at which point a replacement token is sampled from the residual distribution. If all draft tokens are accepted, a bonus
token is sampled.

Although standard speculative decoding provides exact distribution
preservation, its efficiency depends on the number of draft tokens
surviving verification. Moreover, a token can be rejected when
$q_i(x_i)>p_i(x_i)$ even if it remains highly ranked under the target
distribution. NSD addresses this limitation by introducing an
additional acceptance condition based on target nucleus verification.

\paragraph{Nucleus-Based Verification}

NSD augments speculative decoding using target nucleus verification. For each
verification position, we construct a nucleus from the target
distribution $p_i$ using a nucleus threshold $\theta\in(0,1)$:

\begin{equation}
S_i(\theta)=\{v_{i,(1)},\ldots,v_{i,(k_i(\theta))}\}
\quad\text{where}\quad
k_i(\theta)=\min\left\{k:\sum_{j=1}^{k}p_i(v_{i,(j)})\geq\theta\right\}.
\label{eq:nsd_nucleus}
\end{equation}

The nucleus adapts to the target distribution at each position and
is only used as an additional acceptance condition. NSD does not
renormalize $p_i$ or $q_i$ over $S_i(\theta)$. When $q_i(x_i)>p_i(x_i)$, standard speculative decoding may reject $x_i$ even when
$x_i\in S_i(\theta)$. NSD accepts such proposals through the additional
nucleus-based condition: \(
b_i
=
\mathbf{1}\bigl\{
u_i \leq a_i^{\mathrm{SD}}(x_i)
\lor x_i\in S_i(\theta)
\bigr\}.
\)

Algorithm~\ref{alg:nsd} gives one verification cycle. Here $q_i$ is the
actual conditional proposal distribution, and both $p_i$ and $q_i$ include
sampling transformations. Nucleus ties use a fixed ordering. The generation
wrapper stops at EOS or the output budget and discards uncommitted cache
states. The implementation-specific numerical safeguards are described separately in Appendix~\ref{app:method_candidate}.

\begin{algorithm}[H]
\caption{Nucleus Speculative Decoding: one verification cycle}
\label{alg:nsd}
\small
\begin{algorithmic}[1]
\Require Token prefix $h_0$; draft length $K$; nucleus threshold $\theta\in(0,1)$;
         target model $p$; draft model $q$

\Ensure Newly emitted token block $o$ (at most $K+1$ tokens)
\State Generate $x_{1:K}$ with the draft mechanism; retain the actual
       conditional proposal distributions $q_{1:K}$
\State Evaluate $p_i=p(\cdot\mid h_0,x_{<i})$ for $i=1,\ldots,K+1$
       in one target verification pass
\State $o \gets ()$

\For{$i=1,\ldots,K$}
    \State Construct $S_i(\theta)$ from $p_i$ using Equation~\eqref{eq:nsd_nucleus}
    \State Sample $u_i\sim\mathrm{Uniform}(0,1)$; set $a_i^{\mathrm{SD}}(x_i)\gets\min\{1,p_i(x_i)/q_i(x_i)\}$

    \If{$u_i \le a_i^{\mathrm{SD}}(x_i)$ \textbf{or} $x_i \in S_i(\theta)$}
        \State Append $x_i$ to $o$
    \Else
        \State $\rho_i(v) \gets
        \dfrac{[p_i(v)-q_i(v)]_+}
        {\sum_{w\in\mathcal{V}}[p_i(w)-q_i(w)]_+}$
        \State Sample $y\sim\rho_i$; append $y$ to $o$; \Return $o$
    \EndIf
\EndFor

\State Sample bonus token $y\sim p_{K+1}$; append $y$ to $o$; \Return $o$
\end{algorithmic}
\end{algorithm}

\paragraph{Distributional Fidelity}
\label{sec:nsd_local_tradeoff}

We characterize the distributional deviation introduced by
nucleus-based acceptance.
Let \(S=S_\theta(p)\) denote the target nucleus, and let
\(D_{\mathrm{TV}}(r,s)=\frac{1}{2}\sum_v |r(v)-s(v)|\).
Define
\begin{equation}
    R
    =D_{\mathrm{TV}}(q,p)
    =\sum_{v\in\mathcal{V}}[q(v)-p(v)]_+,
    \quad
    E_S
    =\sum_{v\in S}[q(v)-p(v)]_+.
    \label{eq:nsd_excess_mass}
\end{equation}

\(R\) is the rejection probability of standard speculative
decoding, while \(E_S\) is the draft excess probability mass within
the target nucleus. Under a shared proposal \(X\sim q\) under uniform distribution, \(E_S\) is exactly the probability that NSD accepts a proposal rejected by standard SD. Therefore,
\(\bar a_{\mathrm{SD}}=1-R,\;
\bar a_{\mathrm{NSD}}=1-R+E_S\),
and the increase in proposal acceptance is
\(\bar a_{\mathrm{NSD}}-\bar a_{\mathrm{SD}}=E_S\).

For \(R>0\), let \(\rho(v)=[p(v)-q(v)]_+/R\) denote the SD residual distribution.
Since NSD only rejects the draft excess outside the nucleus,
residual resampling decreases from \(R\) to \(R-E_S\), so
\begin{equation}
    \widetilde p(v)
    =
    \min\{p(v),q(v)\}
    +\mathbf{1}\{v\in S\}[q(v)-p(v)]_+
    +(R-E_S)\rho(v).
    \label{eq:nsd_output_distribution}
\end{equation}
Subtracting $p(v)$ gives
\(
    \widetilde p(v)-p(v)
    =
    \mathbf{1}\{v\in S\}[q(v)-p(v)]_+
    -
    E_S\rho(v).
\)
The two terms have disjoint support: the first is nonzero only
where $q(v)>p(v)$, while $\rho(v)$ is supported only where
$p(v)>q(v)$. Hence
\begin{equation}
    D_{\mathrm{TV}}(\widetilde p,p)
    =
    E_S
    =
    \bar a_{\mathrm{NSD}}-\bar a_{\mathrm{SD}}.
    \label{eq:nsd_local_identity}
\end{equation}
Equation~\eqref{eq:nsd_local_identity} shows that the extra acceptance of NSD is exactly its local deviation from the target. Since \(E_S\le R\), this deviation is no larger than the rejection probability of standard SD.

We next give a conditional sequence-level guarantee. Let \(K_t(\cdot\mid h)\) denote the actual next-emitted-token distribution of the decoder, with internal state marginalized conditional on the emitted history \(h\). If \(D_{\mathrm{TV}}(K_t(\cdot\mid h),p_t(\cdot\mid h))\le\epsilon_t\) for every relevant history, then a sequential coupling gives
\[
D_{\mathrm{TV}}\!\left(P_{\mathrm{NSD}}^{(T)},P_p^{(T)}\right)
\le 1-\prod_{t=1}^{T}(1-\epsilon_t).
\]
Applying this bound to a particular speculative backend requires establishing the conditional-kernel assumption, including proposal selection, recovery, bonus sampling, and decoder state.
Equivalently, NSD and target decoding can generate identical sequences
under a coupling with probability at least
\(\prod_{t=1}^{T}(1-\epsilon_t)\). For a uniform per-token error bound
\(\epsilon_t\leq\epsilon\), this becomes
\begin{equation}
\Pr\bigl(X_{1:T}^{\mathrm{NSD}}=X_{1:T}^{\mathrm{target}}\bigr)
\geq (1-\epsilon)^T
\geq 1-T\epsilon .
\end{equation}
Therefore, when the nucleus-based relaxation introduces only small
single-token deviations, NSD remains close to target-model distribution. For illustration, $\epsilon\le10^{-3}$ and $T=32$ imply coupling agreement
at least $(0.999)^{32}\approx0.9685$. Detailed proofs are provided in
Appendix~\ref{app:proofs_candidate}.

\section{Experiments}
\label{sec:experiments}

\subsection{Experimental Setup}

\paragraph{Models.}

We evaluate NSD across five different proposal mechanisms and four model sizes, covering Nemotron-Labs-Diffusion-3B, Nemotron-Labs-Diffusion-8B, Nemotron-Labs-Diffusion-14B \citep{fu2026nemotron}, and Qwen3 models at 4B, 8B, and 14B sizes \citep{yang2025qwen3}. In our main experiments, Nemotron employs linear self-speculation with 32-token proposal blocks, while Qwen3 is paired with the DSpark draft model \citep{cheng2026dspark} using 7-token blocks. To evaluate whether NSD generalizes across different speculative decoding paradigms, we additionally conduct experiments with DFlash \citep{chen2026dflash} and EAGLE-3 \citep{li2025eagle3} on Qwen3, as well as MTP on Qwen3.5 \citep{qwenteam2026qwen35omnitechnicalreport} . These additional results are reported in Appendix~\ref{app:additional_proposals}.

\paragraph{Evaluation.}
We evaluate mathematical reasoning on GSM8K \citep{cobbe2021gsm8k}, MATH500, a subset of MATH \citealp{hendrycks2021}), AIME24, AIME25, and AIME26, as well as functional code generation on MBPP \citep{austin2021} and HumanEval \citep{chen2021}.

\paragraph{Metrics.}
We evaluate both task performance and efficiency. Task performance reports accuracy for mathematical reasoning tasks and functional pass@1 for MBPP and HumanEval. Speedup Ratio is the ratio of each method's token throughput to target model autoregressive (AR) throughput, where AR is normalized to $1.00\times$. The mean accepted length, denoted by $\tau$, is the number of tokens emitted per verification cycle, which includes accepted draft tokens and any resampling or bonus token. So $\tau$ can exceed the draft block size by up to one token.

\paragraph{Baselines.}
We use target model autoregressive decoding as the primary reference for generation quality and latency. We compare Nucleus Sampling (NS), which modifies the sampling distribution through temperature and nucleus truncation. Standard speculative decoding (SD) \citep{leviathan2023,chen2023accelerating} serves as the exact distribution-preserving acceleration baseline. We further compare with relaxed speculative verification methods, including ASD \citep{feng2026asd}, Cactus \citep{hao2026}, and MARS \citep{song2026}, which trade exact verification guarantees for improved acceptance or efficiency. 
As for decoding parameters, all methods use temperature $T=1.0$. For Nucleus Sampling, we use $top-p=0.95$ following the common decoding configuration.
For NSD, we set the nucleus threshold $\theta=0.95$ in the main experiments.

\subsection{Main Results}

\begin{table}[p]
\centering
\caption{Speedup, mean accepted length and accuracy performance overview. AR denotes autoregressive sampling; NS denotes nucleus sampling on AR mode; Diff denotes diffusion-based parallel decoding. Spd.: throughput speedup relative to target-model AR; $\tau$: mean tokens accepted per draft-block verification, including any resampling token; Acc.: math accuracy or code pass@1 (\%). Blue rows denote NSD; bold numbers mark the highest speedup and largest $\tau$ within each comparison.}
\label{tab:speedup}
\begingroup
\small
\setlength{\tabcolsep}{1pt}
\renewcommand{\arraystretch}{1.08}
\begin{adjustbox}{max width=\linewidth}
\begin{tabular}{@{}llccc@{\hspace{8pt}}ccc@{\hspace{8pt}}ccc@{\hspace{8pt}}ccc@{\hspace{8pt}}ccc@{}}
\toprule
& & \multicolumn{3}{c}{GSM8K} & \multicolumn{3}{c}{MATH500} & \multicolumn{3}{c}{AIME} & \multicolumn{3}{c}{HumanEval} & \multicolumn{3}{c}{MBPP} \\
\cmidrule(lr){3-5}\cmidrule(lr){6-8}\cmidrule(lr){9-11}\cmidrule(lr){12-14}\cmidrule(lr){15-17}
Model & Method & Spd. & $\tau$ & Acc. & Spd. & $\tau$ & Acc. & Spd. & $\tau$ & Acc. & Spd. & $\tau$ & Acc. & Spd. & $\tau$ & Acc. \\
\midrule
\textbf{Nemo-3B} & AR & $1.00\times$ & 1.00 & 87.95 & $1.00\times$ & 1.00 & 74.60 & $1.00\times$ & 1.00 & 8.00 & $1.00\times$ & 1.00 & 70.73 & $1.00\times$ & 1.00 & 56.03 \\
 & NS & $0.99\times$ & 1.00 & 88.32 & $0.99\times$ & 1.00 & 76.16 & $0.99\times$ & 1.00 & 10.00 & $0.99\times$ & 1.00 & 74.63 & $0.99\times$ & 1.00 & 57.04 \\
 & Diff & $2.67\times$ & - & 90.07 & $2.59\times$ & - & 73.60 & $2.44\times$ & - & 3.33 & $1.63\times$ & - & 68.90 & $1.52\times$ & - & 56.81 \\
 & SD & $2.73\times$ & 7.30 & 88.63 & $2.17\times$ & 7.49 & 74.52 & $2.05\times$ & 6.01 & 10.00 & $1.66\times$ & 5.61 & 73.41 & $1.52\times$ & 5.04 & 54.86 \\
 & MARS & $3.00\times$ & 9.75 & 85.97 & $3.60\times$ & 9.75 & 73.20 & $3.45\times$ & 9.12 & 11.11 & $2.63\times$ & 6.97 & 74.39 & $2.54\times$ & 6.58 & 54.09 \\
 & ASD & $3.69\times$ & 9.58 & 85.22 & $3.69\times$ & 10.09 & 75.80 & $3.53\times$ & 9.21 & 13.33 & $2.50\times$ & 6.61 & 75.61 & $2.36\times$ & 6.12 & 56.81 \\
 & CACT & $3.87\times$ & 10.17 & 85.37 & $3.82\times$ & 10.79 & 69.80 & $3.68\times$ & 9.54 & 5.56 & $2.90\times$ & 7.77 & 72.56 & $2.78\times$ & 7.26 & 52.14 \\
\rowcolor{sdblue!7}
 & \textbf{NSD} & $\mathbf{3.96\times}$ & \textbf{10.44} & 88.26 & $\mathbf{4.42\times}$ & \textbf{11.51} & 75.96 & $\mathbf{3.95\times}$ & \textbf{10.24} & 13.33 & $\mathbf{3.08\times}$ & \textbf{7.98} & 75.24 & $\mathbf{3.08\times}$ & \textbf{7.52} & 55.25 \\
\addlinespace[2pt]
\midrule
\addlinespace[1pt]
\textbf{Nemo-8B} & AR & $1.00\times$ & 1.00 & 89.84 & $1.00\times$ & 1.00 & 82.20 & $1.00\times$ & 1.00 & 24.44 & $1.00\times$ & 1.00 & 79.76 & $1.00\times$ & 1.00 & 61.09 \\
 & NS & $1.00\times$ & 1.00 & 90.90 & $1.00\times$ & 1.00 & 82.40 & $0.98\times$ & 1.00 & 28.89 & $1.01\times$ & 1.00 & 80.12 & $0.97\times$ & 1.00 & 60.31 \\
 & Diff & $2.02\times$ & - & 93.93 & $3.06\times$ & - & 83.80 & $2.83\times$ & - & 26.67 & $1.68\times$ & - & 82.07 & $1.36\times$ & - & 61.48 \\
 & SD & $2.20\times$ & 5.64 & 93.56 & $1.50\times$ & 5.59 & 87.72 & $1.49\times$ & 5.08 & 27.78 & $1.50\times$ & 4.51 & 81.46 & $1.38\times$ & 4.08 & 59.53 \\
 & MARS & $4.18\times$ & 10.52 & 93.03 & $3.38\times$ & 9.61 & 87.80 & $3.39\times$ & 10.12 & 28.89 & $2.94\times$ & 7.69 & 79.27 & $2.65\times$ & 6.68 & 57.59 \\
 & ASD & $4.08\times$ & 10.37 & 92.70 & $4.19\times$ & 11.90 & 88.60 & $3.59\times$ & 10.48 & 33.33 & $3.43\times$ & 8.82 & 81.10 & $2.92\times$ & 7.39 & 61.87 \\
 & CACT & $4.57\times$ & 12.01 & 88.64 & $4.14\times$ & 11.69 & 82.80 & $3.57\times$ & 10.69 & 30.00 & $3.99\times$ & 10.44 & 80.24 & $3.53\times$ & 9.70 & 58.37 \\
\rowcolor{sdblue!7}
 & \textbf{NSD} & $\mathbf{4.87\times}$ & \textbf{12.21} & 92.57 & $\mathbf{4.40\times}$ & \textbf{13.18} & 88.56 & $\mathbf{4.50\times}$ & \textbf{12.87} & 31.11 & $\mathbf{4.13\times}$ & \textbf{10.98} & 81.95 & $\mathbf{3.83\times}$ & \textbf{9.87} & 62.26 \\
\addlinespace[2pt]
\midrule
\addlinespace[1pt]
\textbf{Nemo-14B} & AR & $1.00\times$ & 1.00 & 93.48 & $1.00\times$ & 1.00 & 84.80 & $1.00\times$ & 1.00 & 33.33 & $1.00\times$ & 1.00 & 80.24 & $1.00\times$ & 1.00 & 60.70 \\
 & NS & $1.01\times$ & 1.00 & 94.23 & $1.00\times$ & 1.00 & 85.00 & $1.00\times$ & 1.00 & 32.22 & $0.99\times$ & 1.00 & 80.85 & $1.00\times$ & 1.00 & 61.48 \\
 & Diff & $2.29\times$ & - & 93.25 & $2.37\times$ & - & 84.80 & $2.93\times$ & - & 20.00 & $1.49\times$ & - & 79.14 & $1.54\times$ & - & 63.04 \\
 & SD & $2.31\times$ & 5.80 & 94.31 & $1.72\times$ & 5.11 & 88.36 & $1.31\times$ & 4.23 & 36.67 & $1.58\times$ & 4.39 & 77.20 & $1.40\times$ & 3.84 & 61.48 \\
 & MARS & $4.73\times$ & 12.00 & 92.80 & $3.82\times$ & 10.24 & 87.40 & $3.45\times$ & 9.99 & 31.11 & $3.17\times$ & 8.00 & 82.32 & $2.72\times$ & 6.73 & 58.37 \\
 & ASD & $4.62\times$ & 11.87 & 91.58 & $4.21\times$ & 11.09 & 86.80 & $3.30\times$ & 9.59 & 34.44 & $3.50\times$ & 8.78 & 81.71 & $2.84\times$ & 6.92 & 59.53 \\
 & CACT & $4.87\times$ & 12.20 & 90.14 & $4.19\times$ & 11.26 & 84.40 & $3.32\times$ & 9.37 & 32.22 & $3.92\times$ & 10.19 & 77.44 & $3.64\times$ & 9.02 & 58.75 \\
\rowcolor{sdblue!7}
 & \textbf{NSD} & $\mathbf{5.16\times}$ & \textbf{12.64} & 92.87 & $\mathbf{4.38\times}$ & \textbf{12.71} & 87.82 & $\mathbf{4.12\times}$ & \textbf{11.36} & 36.67 & $\mathbf{4.24\times}$ & \textbf{11.27} & 81.34 & $\mathbf{3.92\times}$ & \textbf{9.74} & 59.92 \\
\addlinespace[2pt]
\midrule
\addlinespace[1pt]
\textbf{Qwen-4B} & AR & $1.00\times$ & 1.00 & 89.84 & $1.00\times$ & 1.00 & 83.00 & $1.00\times$ & 1.00 & 20.22 & $1.00\times$ & 1.00 & 85.12 & $1.00\times$ & 1.00 & 61.63 \\
 & NS & $0.99\times$ & 1.00 & 89.31 & $0.99\times$ & 1.00 & 84.20 & $1.06\times$ & 1.00 & 14.89 & $0.99\times$ & 1.00 & 85.24 & $0.99\times$ & 1.00 & 62.65 \\
 & SD & $3.73\times$ & 6.18 & 90.90 & $2.87\times$ & 5.63 & 84.36 & $3.16\times$ & 4.39 & 21.11 & $3.01\times$ & 5.46 & 83.90 & $2.90\times$ & 5.14 & 62.96 \\
 & MARS & $4.11\times$ & 6.87 & 91.58 & $3.29\times$ & 6.03 & 84.00 & $3.41\times$ & 5.49 & 18.89 & $3.49\times$ & 5.92 & 81.71 & $3.35\times$ & 5.54 & 62.65 \\
 & ASD & $3.84\times$ & 6.44 & 89.39 & $3.09\times$ & 5.65 & 83.80 & $3.09\times$ & 5.02 & 20.00 & $3.37\times$ & 5.72 & 81.71 & $3.16\times$ & 5.27 & 59.92 \\
 & CACT & $4.12\times$ & 6.99 & 88.62 & $3.63\times$ & 6.68 & 79.00 & $4.09\times$ & 6.35 & 13.33 & $3.74\times$ & 6.30 & 81.10 & $3.57\times$ & 5.90 & 61.48 \\
\rowcolor{sdblue!7}
 & \textbf{NSD} & $\mathbf{4.28\times}$ & \textbf{7.17} & 89.77 & $\mathbf{3.86\times}$ & \textbf{6.84} & 84.10 & $\mathbf{4.13\times}$ & \textbf{6.53} & 20.00 & $\mathbf{3.93\times}$ & \textbf{6.43} & 82.32 & $\mathbf{3.81\times}$ & \textbf{6.02} & 62.10 \\
\addlinespace[2pt]
\midrule
\addlinespace[1pt]
\textbf{Qwen-8B} & AR & $1.00\times$ & 1.00 & 92.04 & $1.00\times$ & 1.00 & 83.60 & $1.00\times$ & 1.00 & 20.00 & $1.00\times$ & 1.00 & 87.80 & $1.00\times$ & 1.00 & 63.81 \\
 & NS & $0.99\times$ & 1.00 & 93.32 & $1.00\times$ & 1.00 & 84.00 & $1.02\times$ & 1.00 & 18.89 & $0.99\times$ & 1.00 & 84.76 & $0.99\times$ & 1.00 & 64.59 \\
 & SD & $4.16\times$ & 6.32 & 91.66 & $3.25\times$ & 5.70 & 84.64 & $3.04\times$ & 4.36 & 17.78 & $3.35\times$ & 5.46 & 87.44 & $3.25\times$ & 5.20 & 66.15 \\
 & MARS & $4.35\times$ & 6.68 & 90.91 & $3.12\times$ & 5.98 & 84.00 & $3.32\times$ & 5.51 & 20.00 & $3.79\times$ & 5.86 & 84.76 & $3.56\times$ & 5.56 & 63.04 \\
 & ASD & $4.26\times$ & 6.50 & 93.56 & $3.36\times$ & 5.63 & 85.60 & $2.94\times$ & 4.92 & 21.11 & $3.59\times$ & 5.54 & 85.36 & $3.52\times$ & 5.35 & 64.12 \\
 & CACT & $4.44\times$ & 7.09 & 91.91 & $4.02\times$ & 6.78 & 82.40 & $3.88\times$ & 6.38 & 12.22 & $4.21\times$ & 6.49 & 85.98 & $4.07\times$ & 6.15 & 64.98 \\
\rowcolor{sdblue!7}
 & \textbf{NSD} & $\mathbf{4.76\times}$ & \textbf{7.24} & 92.30 & $\mathbf{4.31\times}$ & \textbf{6.87} & 86.40 & $\mathbf{4.12\times}$ & \textbf{6.50} & 21.11 & $\mathbf{4.35\times}$ & \textbf{6.58} & 86.46 & $\mathbf{4.14\times}$ & \textbf{6.25} & 64.59 \\
\addlinespace[2pt]
\midrule
\addlinespace[1pt]
\textbf{Qwen-14B} & AR & $1.00\times$ & 1.00 & 93.93 & $1.00\times$ & 1.00 & 87.40 & $1.00\times$ & 1.00 & 22.22 & $1.00\times$ & 1.00 & 85.98 & $1.00\times$ & 1.00 & 66.93 \\
 & NS & $0.99\times$ & 1.00 & 92.80 & $0.99\times$ & 1.00 & 87.60 & $1.00\times$ & 1.00 & 18.89 & $0.99\times$ & 1.00 & 90.24 & $1.00\times$ & 1.00 & 66.54 \\
 & SD & $4.53\times$ & 6.31 & 93.78 & $3.66\times$ & 5.68 & 86.88 & $3.10\times$ & 4.25 & 25.56 & $3.66\times$ & 5.40 & 89.51 & $3.63\times$ & 5.25 & 67.78 \\
 & MARS & $4.81\times$ & 6.74 & 93.56 & $3.98\times$ & 5.95 & 85.60 & $3.40\times$ & 5.37 & 21.11 & $4.05\times$ & 5.74 & 89.63 & $4.02\times$ & 5.59 & 68.09 \\
 & ASD & $4.61\times$ & 6.44 & 93.78 & $3.76\times$ & 5.56 & 86.80 & $2.98\times$ & 4.71 & 26.67 & $3.86\times$ & 5.46 & 89.02 & $3.82\times$ & 5.31 & 68.09 \\
 & CACT & $5.02\times$ & 7.12 & 92.19 & $4.56\times$ & 6.79 & 84.80 & $4.00\times$ & 6.35 & 17.78 & $4.55\times$ & 6.44 & 88.41 & $4.42\times$ & 6.15 & 68.09 \\
\rowcolor{sdblue!7}
 & \textbf{NSD} & $\mathbf{5.15\times}$ & \textbf{7.22} & 93.18 & $\mathbf{4.75\times}$ & \textbf{6.87} & 87.40 & $\mathbf{4.16\times}$ & \textbf{6.53} & 24.44 & $\mathbf{4.64\times}$ & \textbf{6.53} & 90.24 & $\mathbf{4.47\times}$ & \textbf{6.25} & 68.48 \\
\bottomrule
\end{tabular}
\end{adjustbox}
\endgroup
\end{table}

\paragraph{Consistent throughput speedup.}
Table~\ref{tab:speedup} shows that NSD achieves the highest reported
speedup among the evaluated methods in all comparisons. Across models, the mean speedups range from $3.70\times$ to $4.63\times$ relative to AR decoding.
For Nemotron models, NSD achieves overall speedups of
$3.70\times$, $4.35\times$, and $4.36\times$.
The gains also extend to Qwen3 with DSpark drafting:
NSD increases the overall speedups to $4.00\times$, $4.34\times$, and $4.63\times$
for Qwen3-4B, -8B, and -14B respectively.
Averaging the paired throughput ratios across all 30 comparisons gives
mean improvements of $323.0\%$ over AR and $84.98\%$ over SD.
These results demonstrate efficiency gains under both evaluated
proposal mechanisms, different model sizes and diverse model backbones.

\paragraph{Increased mean accepted length.}
The throughput gains coincide with increases in the reported mean accepted length $\tau$ across every model and benchmark group.
For Nemotron-3B, -8B, and -14B, the overall $\tau$ increases from
$6.29$, $4.98$, and $4.67$ under SD to
$9.54$, $11.82$, and $11.54$ under NSD.
For Qwen3-4B, -8B and -14B, the corresponding values increase from
$5.36$, $5.41$ and $5.38$ to $6.60$, $6.69$ and $6.68$.
NSD also achieves the largest reported $\tau$ among the compared
verification methods.
These observations are consistent with NSD retaining more draft tokens.

\paragraph{Task performance trade-offs.}
Task performance varies across models and benchmark groups.
NSD achieves higher reported overall accuracy than SD on five models with increases ranging from
$0.05$ to $1.32$ percentage points.
Qwen3-4B is the exception, where overall accuracy decreases from
$68.65\%$ to $67.66\%$ as the overall speedup increases from
$3.13\times$ to $4.00\times$.
Across the 30 model-benchmark comparisons, NSD exceeds SD in
15 accuracy entries, ties one, and falls below SD in 14.
Across all comparisons, NSD's accuracy difference relative to SD ranges from \(-1.58\) to \(+4.14\) percentage points, with NSD achieving 0.40 percentage points higher accuracy on average. These comparisons show consistent throughput
gains alongside task-dependent quality changes.

\subsection{Ablation Study}

\begin{table}[t]
\centering\footnotesize
\caption{Effect of draft block size $K$ and nucleus threshold $\theta$ on speedup, mean accepted length, and accuracy. Spd. denotes speedup relative to target-model AR; $\tau$ is the mean accepted length per draft block verification; Acc. is math accuracy or code pass@1 (\%).}
\label{tab:topp_ablation}
\begingroup
\setlength{\tabcolsep}{2pt}
\renewcommand{\arraystretch}{1.08}
\resizebox{\linewidth}{!}{%
\begin{tabular}{@{}lcccc@{\hspace{6pt}}ccc@{\hspace{6pt}}ccc@{\hspace{6pt}}ccc@{\hspace{6pt}}ccc@{}}
\toprule
& & \multicolumn{15}{c}{Nucleus threshold $\theta$} \\
\cmidrule(lr){3-17}
& & \multicolumn{3}{c}{0.95} & \multicolumn{3}{c}{0.925} & \multicolumn{3}{c}{0.90} & \multicolumn{3}{c}{0.875} & \multicolumn{3}{c}{0.85} \\
\cmidrule(lr){3-5}\cmidrule(lr){6-8}\cmidrule(lr){9-11}\cmidrule(lr){12-14}\cmidrule(l){15-17}
Dataset & $K$ & Spd. & $\tau$ & Acc. & Spd. & $\tau$ & Acc. & Spd. & $\tau$ & Acc. & Spd. & $\tau$ & Acc. & Spd. & $\tau$ & Acc. \\
\midrule
\multicolumn{17}{l}{\textbf{Nemo-8B}} \\
GSM8K & 8 & $2.92\times$ & 6.61 & 91.66 & $2.89\times$ & 6.61 & 91.66 & $2.83\times$ & 6.67 & 91.66 & $2.69\times$ & 6.51 & 92.42 & $2.65\times$ & 6.54 & 92.87 \\
 & 16 & $4.35\times$ & 10.70 & 91.58 & $3.63\times$ & 9.42 & 93.48 & $4.21\times$ & 9.66 & 93.94 & $3.93\times$ & 9.56 & 93.10 & $4.06\times$ & 9.46 & 92.72 \\
 & 24 & $4.55\times$ & 11.90 & 92.87 & $4.60\times$ & 11.74 & 90.90 & $4.65\times$ & 12.09 & 90.14 & $4.70\times$ & 12.03 & 90.98 & $4.89\times$ & 12.47 & 91.58 \\
 & 32 & $4.87\times$ & 12.21 & 92.57 & $5.04\times$ & 13.03 & 91.43 & $5.36\times$ & 13.57 & 93.63 & $4.88\times$ & 12.26 & 91.74 & $5.30\times$ & 13.53 & 92.42 \\
\addlinespace[5pt]
MBPP & 8 & $2.72\times$ & 6.17 & 61.48 & $2.53\times$ & 6.23 & 59.53 & $2.70\times$ & 6.27 & 60.70 & $2.48\times$ & 6.04 & 59.14 & $2.57\times$ & 6.20 & 62.26 \\
 & 16 & $3.52\times$ & 8.15 & 59.92 & $3.47\times$ & 7.93 & 59.92 & $3.43\times$ & 7.82 & 58.75 & $3.30\times$ & 7.56 & 60.31 & $3.26\times$ & 7.77 & 62.65 \\
 & 24 & $3.79\times$ & 9.79 & 58.75 & $3.92\times$ & 9.67 & 58.37 & $3.77\times$ & 9.48 & 59.53 & $3.94\times$ & 9.66 & 58.75 & $3.65\times$ & 8.88 & 62.26 \\
 & 32 & $3.83\times$ & 9.87 & 62.26 & $3.88\times$ & 9.46 & 60.63 & $3.95\times$ & 9.66 & 59.92 & $3.77\times$ & 9.06 & 61.48 & $3.61\times$ & 8.88 & 61.88 \\
\addlinespace[5pt]
\midrule
\multicolumn{17}{l}{\textbf{Nemo-14B}} \\
GSM8K & 8 & $2.91\times$ & 6.65 & 90.90 & $2.97\times$ & 6.81 & 91.74 & $2.93\times$ & 6.74 & 92.42 & $2.70\times$ & 6.54 & 93.18 & $2.98\times$ & 6.61 & 94.01 \\
 & 16 & $3.82\times$ & 9.33 & 90.14 & $4.04\times$ & 9.83 & 91.58 & $3.99\times$ & 9.79 & 92.95 & $4.16\times$ & 10.08 & 93.03 & $4.26\times$ & 10.29 & 93.78 \\
 & 24 & $4.66\times$ & 11.40 & 90.45 & $4.67\times$ & 11.12 & 92.57 & $4.25\times$ & 10.62 & 92.04 & $4.42\times$ & 10.91 & 92.95 & $4.56\times$ & 11.19 & 93.48 \\
 & 32 & $5.16\times$ & 12.64 & 92.87 & $5.00\times$ & 12.53 & 92.80 & $5.11\times$ & 12.31 & 92.70 & $4.76\times$ & 11.69 & 93.10 & $4.86\times$ & 11.81 & 93.86 \\
\addlinespace[5pt]
MBPP & 8 & $2.59\times$ & 6.11 & 58.37 & $2.60\times$ & 6.29 & 59.53 & $2.56\times$ & 5.95 & 63.42 & $2.56\times$ & 5.92 & 62.26 & $2.56\times$ & 5.95 & 63.81 \\
 & 16 & $3.40\times$ & 8.16 & 60.70 & $3.32\times$ & 7.89 & 61.48 & $3.39\times$ & 8.10 & 62.57 & $3.31\times$ & 7.91 & 61.48 & $3.25\times$ & 8.03 & 64.20 \\
 & 24 & $3.68\times$ & 9.08 & 59.14 & $3.81\times$ & 9.53 & 61.87 & $3.57\times$ & 8.73 & 61.87 & $3.37\times$ & 8.15 & 60.70 & $3.53\times$ & 8.63 & 63.04 \\
 & 32 & $3.92\times$ & 9.74 & 59.92 & $3.59\times$ & 8.87 & 60.31 & $3.59\times$ & 8.86 & 62.65 & $3.70\times$ & 9.19 & 61.87 & $3.58\times$ & 8.89 & 62.26 \\
\addlinespace[5pt]
\bottomrule
\end{tabular}
}
\endgroup

\end{table}

We vary the nucleus threshold $\theta$ over five values: $0.85$, $0.875$,
$0.90$, $0.925$, and $0.95$. We also vary the draft block length $K$
over $8$, $16$, $24$, and $32$ tokens.
Table~\ref{tab:topp_ablation} reports the joint effects on decoding
efficiency and task performance for Nemotron-8B and Nemotron-14B
on GSM8K and MBPP.

\paragraph{Effect of draft block length.}
Increasing $K$ from $8$ to $32$ improves both speedup and mean
acceptance length for both models and benchmarks at every tested
nucleus threshold.
At $\theta=0.95$, the GSM8K speedup increases from
$2.92\times$ to $4.87\times$ for Nemotron-8B and from
$2.91\times$ to $5.16\times$ for Nemotron-14B.
The corresponding mean accepted lengths increase from
$6.61$ to $12.21$ and from $6.65$ to $12.64$.
On MBPP, the same block-length change increases speedup from
$2.72\times$ to $3.83\times$ for Nemotron-8B and from
$2.59\times$ to $3.92\times$ for Nemotron-14B.
These results show that longer draft blocks allow NSD to retain
more tokens per verification. The larger speedup gains on GSM8K than on MBPP further suggest that the benefit of increasing K varies across tasks.

\paragraph{Diminishing returns from longer blocks.}
The marginal benefit of increasing $K$ depends on the task.
For GSM8K, increasing $K$ from $24$ to $32$ improves throughput
for both models at all thresholds.
However, the gains often diminish on MBPP.
For Nemotron-8B at $\theta=0.95$, this increase changes speedup
only from $3.79\times$ to $3.83\times$ and $\tau$ from
$9.79$ to $9.87$.
For Nemotron-14B at $\theta=0.925$, speedup instead decreases
from $3.81\times$ to $3.59\times$, accompanied by a decrease
in $\tau$ from $9.53$ to $8.87$.
Thus, extending the proposed block does not always
increase $\tau$ or improve throughput.

\paragraph{Sensitivity to the nucleus threshold.}
The effect of $\theta$ is model- and task-dependent rather than
uniformly monotonic.
For Nemotron-14B at $K=32$, reducing $\theta$ from $0.95$ to $0.85$
increases GSM8K accuracy from $92.87\%$ to $93.86\%$ and
MBPP pass@1 from $59.92\%$ to $62.26\%$.
These improvements are accompanied by lower speedups:
$5.16\times$ to $4.86\times$ on GSM8K and
$3.92\times$ to $3.58\times$ on MBPP.
In contrast, for Nemotron-8B on GSM8K at $K=32$,
reducing $\theta$ from $0.95$ to $0.90$ improves both speedup
and accuracy, from $4.87\times$ to $5.36\times$ and from
$92.57\%$ to $93.63\%$, while increasing $\tau$ from
$12.21$ to $13.57$.
The same threshold change on MBPP increases speedup from
$3.83\times$ to $3.95\times$ but decreases pass@1 from
$62.26\%$ to $59.92\%$.
A threshold adjustment that improves both metrics on one task
can therefore introduce a quality-efficiency trade-off on another.

\paragraph{Joint configuration trade-offs.}
The configurations maximizing throughput do not necessarily
maximize task performance.
Although the highest observed speedup occurs at $K=32$, the highest MBPP pass@1 for both models
occurs at $K=16$ and $\theta=0.85$.
This setting achieves $62.65\%$ pass@1 with a $3.26\times$
speedup for Nemotron-8B and $64.20\%$ pass@1 with a
$3.25\times$ speedup for Nemotron-14B.
Taken together, the ablations show that block length and nucleus
threshold should be considered jointly, as larger blocks generally
favor throughput by increasing the number of draft tokens that can be
retained per verification step, whereas the preferred nucleus threshold
depends on the model and task and can yield a more favorable balance
between efficiency and performance.

\subsection{Analysis}

\begin{figure}[t]
\centering
\includegraphics[width=\linewidth]{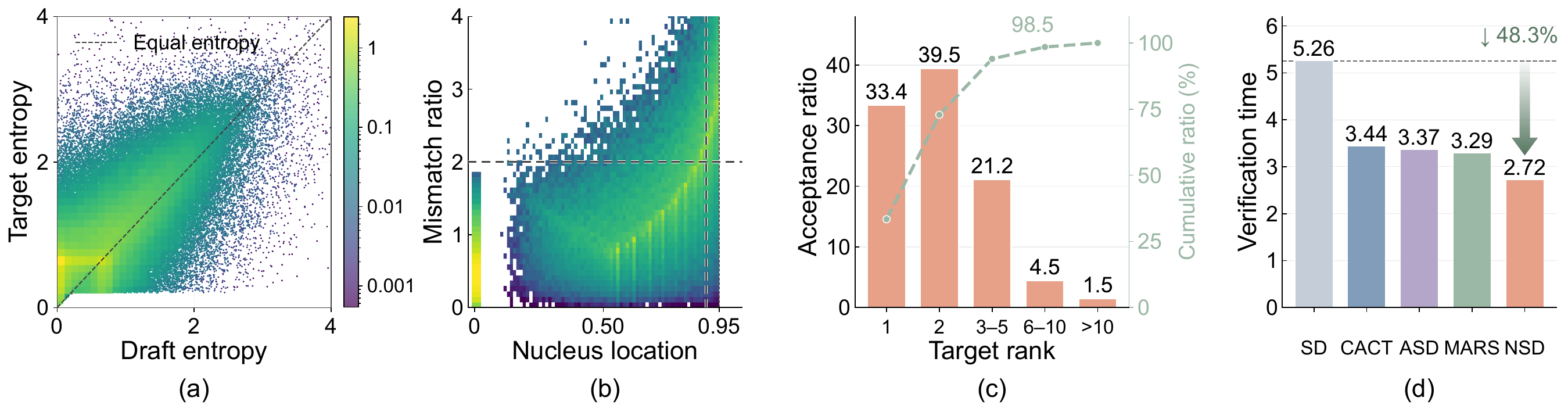}
\caption{Nucleus-based acceptance statistics and verification efficiency.
(a) Draft and target entropy.
(b) Nucleus location and log probability mismatch; location zero denotes the target rank-1 token, with its strip widened for visibility.
(c) Nucleus-accepted token shares and their cumulative ratio by target rank.
(d) Verification time in ms/output token for Nemotron models.}
\label{fig:rescue_target_draft_entropy}
\label{fig:qwen3_4b_rescue_triptych}
\end{figure}

We investigate the statistics of additional acceptances under NSD
and how this relaxation affects verification efficiency.
An additional acceptance is a verified draft proposal that fails
the standard stochastic acceptance but is retained by NSD because
it belongs to the target nucleus.
Figure~\ref{fig:qwen3_4b_rescue_triptych}(a--c) characterizes additional acceptances from Qwen3-4B with DSpark. Figure~\ref{fig:qwen3_4b_rescue_triptych}(d) compares verification timing on Nemotron-14B with self-speculation.

\paragraph{Draft distributions are often more concentrated.}
Figure~\ref{fig:qwen3_4b_rescue_triptych}(a) compares draft and target entropy at contexts where
additional acceptances occur.
The mean draft entropy is 0.865, compared with 1.101 for
the target, and target entropy exceeds draft entropy in 70.30\%
of these tokens.
The two entropies are positively correlated, with Pearson
correlation $r=0.668$.
These observations indicate that although draft and target
entropy tend to vary together, the draft distribution is often
more concentrated. This pattern complements the target-rank analysis by
distinguishing distributional concentration.
For an additionally accepted token $x$, failure of the
standard SD verification requires $q(x)>p(x)$, with rejection
probability $1-p(x)/q(x)$.
Thus, a proposal can be highly ranked by the target yet still
face a substantial rejection probability when the draft assigns
it more probability.

\paragraph{Most additional acceptances lie inside the nucleus core.}
Figure~\ref{fig:qwen3_4b_rescue_triptych}(b) further characterizes additional acceptances along two dimensions: nucleus location and target-draft probability mismatch.
Let $C(x)$ denote the cumulative target probability mass preceding
token $x$, and let
$Y(x)=\log(q(x)/p(x))$.
Within the additionally accepted tokens, we distinguish the nucleus core, $C(x)<0.9$, and the boundary region $C(x)\geq 0.9$.
Splitting the mismatch dimension at $Y=2$ yields four categories: core-lower mismatch, core-higher mismatch,
boundary-lower mismatch, and boundary-higher mismatch. Overall, 86.65\% of additional acceptances lie in the nucleus core,
showing that the relaxation does not operate primarily on tokens
near the nucleus boundary. The available statistics also show that 79.30\% satisfy both
$C(x)<0.9$ and the mismatch condition $Y(x)<2$. The remaining tokens are distributed among the
core-higher mismatch (\(7.36\%\)), boundary-lower mismatch (\(4.84\%\)),
and boundary-higher mismatch (\(8.50\%\)) regions. This concentration provides further evidence that nucleus-based
acceptance predominantly retains plausible tokens under the target distribution. And the draft-target mismatch is not significant.

\paragraph{Additional acceptances concentrate at high target ranks.}
Figure~\ref{fig:qwen3_4b_rescue_triptych}(c) shows that NSD frequently accepts proposals ranked highly
by the target model.
Among the additionally accepted tokens, 33.4\% are the target's rank-1 token, and 72.9\% fall within its top two.
The cumulative share reaches 94.1\% for the top five and 98.5\% for the top ten.
The results illustrate that high target rank does not ensure
acceptance under SD: its stochastic verification depends on
$\min\{1,p(x)/q(x)\}$ rather than target rank.
Thus, nucleus-based acceptance retains many highly ranked 
tokens that would otherwise be rejected to preserve the target distribution.

\paragraph{Substantial savings in verification time.}
In the verification timing comparison on Nemotron models,
NSD reduces verification time from 5.26 to 2.72 milliseconds per
output token, a 48.3\% reduction relative to SD. Compared with Cactus, ASD, and MARS, whose corresponding costs
are 3.44, 3.37, and 3.29 milliseconds, NSD reduces verification
time per output token by 20.9\%, 19.3\%, and 17.3\% respectively.
Verification calls per output token decrease from $0.1600$ under SD
 to $0.0835$ under NSD, a $47.81\%$ reduction, while target execution time
 per call changes by less than $1\%$ ($32.551$ versus $32.266$ milliseconds).
 These measurements support fewer verification calls per emitted token as
 the primary source of savings. NSD can continue accepting subsequent proposals
 using already computed target probabilities, spreading expensive verification cost over more tokens.

Taken together, these results clarify how nucleus-based verification
translates additional acceptance into end-to-end efficiency.
The additional proposals retained by NSD are not concentrated near the
boundary of the target nucleus: 86.65\% lie in the nucleus core, and 94.1\%
rank within the target model's top five tokens.
This reduces the frequency of expensive target verification. Verification calls per output token decrease by 47.81\%.
Therefore, the primary efficiency gain of NSD comes from amortizing
each target verification over more emitted tokens rather than making
an individual target-model invocation cheaper.
Competitive accuracy results suggest that NSD improves decoding efficiency by
relaxing verification mainly for highly plausible proposals under the target model.
\FloatBarrier
\section{Conclusion}

In this work, we introduce Nucleus Speculative Decoding (NSD), a simple but effective
relaxed verification method that supplements stochastic acceptance with
target-nucleus membership. It does not modify the proposal mechanism, residual
correction, or bonus sampling procedure, allowing it to be combined with diverse
speculative decoding paradigms. We provide a theoretical analysis of the distributional deviation. The deviation from the target distribution is exactly the
draft model's excess probability inside the target nucleus for a
single decoding step. This derivation reveals when NSD remains exact and provides
conditional sequence-level fidelity bounds relative to the target distribution. Experiments across Qwen and Nvidia-Nemotron model families and proposal mechanisms demonstrate
that our method improves throughput by retaining additional draft tokens while
maintaining competitive task performance. Results and analysis suggest that
plausibility-aware verification is a promising direction for designing
efficient speculative decoding algorithms. 
\label{maintextend}


\section*{Acknowledgements}
We thank the EIT and IDT High Performance Computing Center for providing computational resources for this project.

\begingroup
\raggedright
\bibliography{references,sections/new_related_work}

@misc{cheng2026dspark,
      title={DSpark: Confidence-Scheduled Speculative Decoding with Semi-Autoregressive Generation}, 
      author={Xin Cheng and Xingkai Yu and Chenze Shao and Jiashi Li and Yunfan Xiong and Yi Qian and Jiaqi Zhu and Shirong Ma and Xiaokang Zhang and Jiasheng Ye and Qinyu Chen and Chengqi Deng and Jiping Yu and Damai Dai and Zhengyan Zhang and Yixuan Wei and Yixuan Tan and Wenkai Yang and Runxin Xu and Yu Wu and Zhean Xu and Xuanyu Wang and Muyang Chen and Rui Tian and Xiao Bi and Zhewen Hao and Shaoyuan Chen and Huanqi Cao and Wentao Zhang and Anyi Xu and Huishuai Zhang and Dongyan Zhao and Wenfeng Liang},
      year={2026},
      eprint={2607.05147},
      archivePrefix={arXiv},
      primaryClass={cs.AI},
      url={https://arxiv.org/abs/2607.05147}, 
}

@misc{chen2026dflash,
      title={DFlash: Block Diffusion for Flash Speculative Decoding}, 
      author={Jian Chen and Yesheng Liang and Zhijian Liu},
      year={2026},
      eprint={2602.06036},
      archivePrefix={arXiv},
      primaryClass={cs.CL},
      url={https://arxiv.org/abs/2602.06036}, 
}

@misc{su2022welm,
      title={WeLM: A Well-Read Pre-trained Language Model for Chinese}, 
      author={Hui Su and Xiao Zhou and Houjin Yu and Xiaoyu Shen and Yuwen Chen and Zilin Zhu and Yang Yu and Jie Zhou},
      year={2022},
      eprint={2209.10372},
      archivePrefix={arXiv},
      primaryClass={cs.CL},
      url={https://arxiv.org/abs/2209.10372}, 
}

@misc{fu2026nemotron,
      title={Nemotron-Labs-Diffusion: A Tri-Mode Language Model Unifying Autoregressive, Diffusion, and Self-Speculation Decoding}, 
      author={Yonggan Fu and Lexington Whalen and Abhinav Garg and Chengyue Wu and Maksim Khadkevich and Nicolai Oswald and Enze Xie and Daniel Egert and Sharath Turuvekere Sreenivas and Shizhe Diao and Chenhan Yu and Ye Yu and Weijia Chen and Sajad Norouzi and Jingyu Liu and Shiyi Lan and Ligeng Zhu and Jin Wang and Jindong Jiang and Morteza Mardani and Mehran Maghoumi and Song Han and Ante Jukić and Nima Tajbakhsh and Jan Kautz and Pavlo Molchanov},
      year={2026},
      eprint={2607.05722},
      archivePrefix={arXiv},
      primaryClass={cs.CL},
      url={https://arxiv.org/abs/2607.05722}, 
}

@misc{bachmann2025,
      title={Judge Decoding: Faster Speculative Sampling Requires Going Beyond Model Alignment}, 
      author={Gregor Bachmann and Sotiris Anagnostidis and Albert Pumarola and Markos Georgopoulos and Artsiom Sanakoyeu and Yuming Du and Edgar Schönfeld and Ali Thabet and Jonas Kohler},
      year={2025},
      eprint={2501.19309},
      archivePrefix={arXiv},
      primaryClass={cs.LG},
      url={https://arxiv.org/abs/2501.19309}, 
}

@misc{holsman2025,
      title={Fuzzy Speculative Decoding for a Tunable Accuracy-Runtime Tradeoff}, 
      author={Maximilian Holsman and Yukun Huang and Bhuwan Dhingra},
      year={2025},
      eprint={2502.20704},
      archivePrefix={arXiv},
      primaryClass={cs.AI},
      url={https://arxiv.org/abs/2502.20704}, 
}

@misc{sun2026,
      title={Revisiting Judge Decoding from First Principles via Training-Free Distributional Divergence}, 
      author={Shengyin Sun and Yiming Li and Renxi Liu and Weizhe Lin and Hui-Ling Zhen and Xianzhi Yu and Mingxuan Yuan and Chen Ma},
      year={2026},
      eprint={2601.04766},
      archivePrefix={arXiv},
      primaryClass={cs.CL},
      url={https://arxiv.org/abs/2601.04766}, 
}

@misc{chen2021,
      title={Evaluating Large Language Models Trained on Code}, 
      author={Mark Chen and Jerry Tworek and Heewoo Jun and Qiming Yuan and Henrique Ponde de Oliveira Pinto and Jared Kaplan and Harri Edwards and Yuri Burda and Nicholas Joseph and Greg Brockman and Alex Ray and Raul Puri and Gretchen Krueger and Michael Petrov and Heidy Khlaaf and Girish Sastry and Pamela Mishkin and Brooke Chan and Scott Gray and Nick Ryder and Mikhail Pavlov and Alethea Power and Lukasz Kaiser and Mohammad Bavarian and Clemens Winter and Philippe Tillet and Felipe Petroski Such and Dave Cummings and Matthias Plappert and Fotios Chantzis and Elizabeth Barnes and Ariel Herbert-Voss and William Hebgen Guss and Alex Nichol and Alex Paino and Nikolas Tezak and Jie Tang and Igor Babuschkin and Suchir Balaji and Shantanu Jain and William Saunders and Christopher Hesse and Andrew N. Carr and Jan Leike and Josh Achiam and Vedant Misra and Evan Morikawa and Alec Radford and Matthew Knight and Miles Brundage and Mira Murati and Katie Mayer and Peter Welinder and Bob McGrew and Dario Amodei and Sam McCandlish and Ilya Sutskever and Wojciech Zaremba},
      year={2021},
      eprint={2107.03374},
      archivePrefix={arXiv},
      primaryClass={cs.LG},
      url={https://arxiv.org/abs/2107.03374}, 
}

@misc{austin2021,
      title={Program Synthesis with Large Language Models}, 
      author={Jacob Austin and Augustus Odena and Maxwell Nye and Maarten Bosma and Henryk Michalewski and David Dohan and Ellen Jiang and Carrie Cai and Michael Terry and Quoc Le and Charles Sutton},
      year={2021},
      eprint={2108.07732},
      archivePrefix={arXiv},
      primaryClass={cs.PL},
      url={https://arxiv.org/abs/2108.07732}, 
}

@misc{hendrycks2021,
      title={Measuring Mathematical Problem Solving With the MATH Dataset}, 
      author={Dan Hendrycks and Collin Burns and Saurav Kadavath and Akul Arora and Steven Basart and Eric Tang and Dawn Song and Jacob Steinhardt},
      year={2021},
      eprint={2103.03874},
      archivePrefix={arXiv},
      primaryClass={cs.LG},
      url={https://arxiv.org/abs/2103.03874}, 
}

@misc{cobbe2021gsm8k,
      title={Training Verifiers to Solve Math Word Problems}, 
      author={Karl Cobbe and Vineet Kosaraju and Mohammad Bavarian and Mark Chen and Heewoo Jun and Lukasz Kaiser and Matthias Plappert and Jerry Tworek and Jacob Hilton and Reiichiro Nakano and Christopher Hesse and John Schulman},
      year={2021},
      eprint={2110.14168},
      archivePrefix={arXiv},
      primaryClass={cs.LG},
      url={https://arxiv.org/abs/2110.14168}, 
}

@misc{yang2025qwen3,
      title={Qwen3 Technical Report}, 
      author={An Yang and Anfeng Li and Baosong Yang and Beichen Zhang and Binyuan Hui and Bo Zheng and Bowen Yu and Chang Gao and Chengen Huang and Chenxu Lv and Chujie Zheng and Dayiheng Liu and Fan Zhou and Fei Huang and Feng Hu and Hao Ge and Haoran Wei and Huan Lin and Jialong Tang and Jian Yang and Jianhong Tu and Jianwei Zhang and Jianxin Yang and Jiaxi Yang and Jing Zhou and Jingren Zhou and Junyang Lin and Kai Dang and Keqin Bao and Kexin Yang and Le Yu and Lianghao Deng and Mei Li and Mingfeng Xue and Mingze Li and Pei Zhang and Peng Wang and Qin Zhu and Rui Men and Ruize Gao and Shixuan Liu and Shuang Luo and Tianhao Li and Tianyi Tang and Wenbiao Yin and Xingzhang Ren and Xinyu Wang and Xinyu Zhang and Xuancheng Ren and Yang Fan and Yang Su and Yichang Zhang and Yinger Zhang and Yu Wan and Yuqiong Liu and Zekun Wang and Zeyu Cui and Zhenru Zhang and Zhipeng Zhou and Zihan Qiu},
      year={2025},
      eprint={2505.09388},
      archivePrefix={arXiv},
      primaryClass={cs.CL},
      url={https://arxiv.org/abs/2505.09388}, 
}

@misc{li2024eagle,
      title={EAGLE: Speculative Sampling Requires Rethinking Feature Uncertainty}, 
      author={Yuhui Li and Fangyun Wei and Chao Zhang and Hongyang Zhang},
      year={2025},
      eprint={2401.15077},
      archivePrefix={arXiv},
      primaryClass={cs.LG},
      url={https://arxiv.org/abs/2401.15077}, 
}

@misc{li2024eagle2,
      title={EAGLE-2: Faster Inference of Language Models with Dynamic Draft Trees}, 
      author={Yuhui Li and Fangyun Wei and Chao Zhang and Hongyang Zhang},
      year={2024},
      eprint={2406.16858},
      archivePrefix={arXiv},
      primaryClass={cs.CL},
      url={https://arxiv.org/abs/2406.16858}, 
}

@misc{gloeckle2024mtp,
      title={Better \& Faster Large Language Models via Multi-token Prediction}, 
      author={Fabian Gloeckle and Badr Youbi Idrissi and Baptiste Rozière and David Lopez-Paz and Gabriel Synnaeve},
      year={2024},
      eprint={2404.19737},
      archivePrefix={arXiv},
      primaryClass={cs.CL},
      url={https://arxiv.org/abs/2404.19737}, 
}

@misc{deepseekai2026deepseekv4,
      title={DeepSeek-V4: Towards Highly Efficient Million-Token Context Intelligence}, 
      author={DeepSeek-AI and Anyi Xu and Bangcai Lin and Bing Xue and Bingxuan Wang and Bingzheng Xu and Bochao Wu and Bowei Zhang and Chaofan Lin and Chen Dong and Chenchen Ling and Chengda Lu and Chenggang Zhao and Chengqi Deng and Chengyu Hou and Chenhao Xu and Chenze Shao and Chong Ruan and Conner Sun and Damai Dai and Daya Guo and Dejian Yang and Deli Chen and Donghao Li and Dongjie Ji and Erhang Li and Fang Wei and Fangyun Lin and Fangzhou Yuan and Feiyu Xia and Fucong Dai and Guangbo Hao and Guanting Chen and Guoai Cao and Guolai Meng and Guowei Li and Han Yu and Han Zhang and Hanwei Xu and Hao Li and Haofen Liang and Haoling Zhang and Haoming Luo and Haoran Wei and Haotian Yuan and Haowei Zhang and Haowen Luo and Haoyu Chen and Haozhe Ji and Hengqing Zhang and Honghui Ding and Hongxuan Tang and Huanqi Cao and Huazuo Gao and Hui Qu and Hui Zeng and J Yang and JQ Zhu and Jia Luo and Jia Song and Jia Yu and Jialiang Huang and Jialu Cai and Jian Liang and Jiangting Zhou and Jiasheng Ye and Jiashi Li and Jiaxin Xu and Jiewen Hu and Jieyu Yang and Jin Chen and Jin Yan and Jingchang Chen and Jingli Zhou and Jingting Xiang and Jingyang Yuan and Jingyuan Cheng and Jingzi Zhou and Jinhua Zhu and Jiping Yu and Joseph Sun and Jun Ran and Junguang Jiang and Junjie Qiu and Junlong Li and Junmin Zheng and Junxiao Song and Kai Dong and Kaige Gao and Kang Guan and Kexing Zhou and Kezhao Huang and Kuai Yu and Lean Wang and Lecong Zhang and Lei Wang and Leyi Xia and Li Zhang and Liang Zhao and Lihua Guo and Lingxiao Luo and Linwang Ma and Linyan Zhu and Litong Wang and Liyu Cai and Liyue Zhang and Longhao Chen and MS Di and MY Xu and Max Mei and Miaojun Wang and Mingchuan Zhang and Minghua Zhang and Minghui Tang and Mingming Li and Mingxu Zhou and Minmin Han and Ning Wang and Panpan Huang and Panpan Wang and Peixin Cong and Peiyi Wang and Peng Zhang and Qiancheng Wang and Qihao Zhu and Qingyang Li and Qinyu Chen and Qiushi Du and Qiwei Jiang and Rui Tian and Ruifan Xu and Ruijie Lu and Ruiling Xu and Ruiqi Ge and Ruisong Zhang and Ruizhe Pan and Runji Wang and Runqian Chen and Runqiu Yin and Runxin Xu and Ruomeng Shen and Ruoyu Zhang and Ruyi Chen and SH Liu and Shanghao Lu and Shangmian Sun and Shangyan Zhou and Shanhuang Chen and Shaofei Cai and Shaoheng Nie and Shaoqing Wu and Shaoyuan Chen and Shengding Hu and Shengyu Liu and Shiqiang Hu and Shirong Ma and Shiyu Wang and Shuiping Yu and Shunfeng Zhou and Shuting Pan and Shuying Yu and Songyang Zhou and Tao Ni and Tao Yun and Tian Jin and Tian Pei and Tian Ye and Tianle Lin and Tianran Ji and Tianyi Cui and Tianyuan Yue and Tingting Yu and Tun Wang and W Zhang and WL Xiao and Wangding Zeng and Wei An and Weilin Zhao and Wen Liu and Wenfeng Liang and Wenjie Pang and Wenjing Luo and Wenjing Yao and Wenjun Gao and Wenkai Yang and Wenlve Huang and Wenqing Hou and Wentao Zhang and Wenting Ma and Xi Gao and Xiang He and Xiangwen Wang and Xianzu Wang and Xiao Bi and Xiaodong Liu and Xiaohan Wang and Xiaokang Chen and Xiaokang Zhang and Xiaotao Nie and Xiaowen Sun and Xiaoxiang Wang and Xin Cheng and Xin Liu and Xin Xie and Xingchao Liu and Xingchen Liu and Xingkai Yu and Xingyou Li and Xinyu Yang and Xinyu Zhang and Xu Chen and Xuanyu Wang and Xuecheng Su and Xueyin Chen and Xuheng Lin and Xuwei Fu and YC Yan and YQ Wang and YW Ma and Yanfeng Luo and Yang Zhang and Yanhong Xu and Yanru Ma and Yanwen Huang and Yao Li and Yao Li and Yao Xu and Yao Zhao and Yaofeng Sun and Yaohui Wang and Yi Qian and Yi Shao and Yi Yu and Yichao Zhang and Yifan Ding and Yifan Shi and Yijia Wu and Yiliang Xiong and Yiling Ma and Ying He and Ying Tang and Ying Zhou and Yingjia Luo and Yinmin Zhong and Yishi Piao and Yisong Wang and Yixiang Zhang and Yixiao Chen and Yixuan Tan and Yixuan Wei and Yiyang Ma and Yiyuan Liu and Yonglun Yang and Yongqiang Guo and Yongtong Wu and Yu Wu and YuKun Li and Yuan Cheng and Yuan Ou and Yuanfan Xu and Yuanhao Li and Yuduan Wang and Yuehan Yang and Yuer Xu and Yuhan Wu and Yuhao Meng and Yuheng Zou and Yukun Zha and Yunfan Xiong and Yupeng Chen and Yuping Lin and Yuqian Cao and Yuqian Wang and Yushun Zhang and Yuting Yan and Yutong Lin and Yuxian Gu and Yuxiang Luo and Yuxiang You and Yuxuan Liu and Yuxuan Zhou and Yuyang Zhou and Yuzhen Huang and ZF Wu and Zehao Wang and Zehua Zhao and Zehui Ren and Zekai Zhang and Zhangli Sha and Zhe Fu and Zhe Ju and Zhean Xu and Zhenda Xie and Zhengyan Zhang and Zheren Gao and Zhewen Hao and Zhibin Gou and Zhicheng Ma and Zhigang Yan and Zhihong Shao and Zhixian Huang and Zhixuan Chen and Zhiyu Wu and Zhizhou Ren and Zhongyu Wu and Zhuoshu Li and Zhuping Zhang and Zian Xu and Zihao Wang and Zihua Qu and Zihui Gu and Zijia Zhu and Zilin Li and Zipeng Zhang and Ziwei Xie and Ziyi Gao and Ziyi Wan and Zizheng Pan and Zongqing Yao},
      year={2026},
      eprint={2606.19348},
      archivePrefix={arXiv},
      primaryClass={cs.CL},
      url={https://arxiv.org/abs/2606.19348}, 
}

@article{guo2025r1,
   title={DeepSeek-R1 incentivizes reasoning in LLMs through reinforcement learning},
   volume={645},
   ISSN={1476-4687},
   url={http://dx.doi.org/10.1038/s41586-025-09422-z},
   DOI={10.1038/s41586-025-09422-z},
   number={8081},
   journal={Nature},
   publisher={Springer Science and Business Media LLC},
   author={Guo, Daya and Yang, Dejian and Zhang, Haowei and Song, Junxiao and Wang, Peiyi and Zhu, Qihao and Xu, Runxin and Zhang, Ruoyu and Ma, Shirong and Bi, Xiao and Zhang, Xiaokang and Yu, Xingkai and Wu, Yu and Wu, Z. F. and Gou, Zhibin and Shao, Zhihong and Li, Zhuoshu and Gao, Ziyi and Liu, Aixin and Xue, Bing and Wang, Bingxuan and Wu, Bochao and Feng, Bei and Lu, Chengda and Zhao, Chenggang and Deng, Chengqi and Ruan, Chong and Dai, Damai and Chen, Deli and Ji, Dongjie and Li, Erhang and Lin, Fangyun and Dai, Fucong and Luo, Fuli and Hao, Guangbo and Chen, Guanting and Li, Guowei and Zhang, H. and Xu, Hanwei and Ding, Honghui and Gao, Huazuo and Qu, Hui and Li, Hui and Guo, Jianzhong and Li, Jiashi and Chen, Jingchang and Yuan, Jingyang and Tu, Jinhao and Qiu, Junjie and Li, Junlong and Cai, J. L. and Ni, Jiaqi and Liang, Jian and Chen, Jin and Dong, Kai and Hu, Kai and You, Kaichao and Gao, Kaige and Guan, Kang and Huang, Kexin and Yu, Kuai and Wang, Lean and Zhang, Lecong and Zhao, Liang and Wang, Litong and Zhang, Liyue and Xu, Lei and Xia, Leyi and Zhang, Mingchuan and Zhang, Minghua and Tang, Minghui and Zhou, Mingxu and Li, Meng and Wang, Miaojun and Li, Mingming and Tian, Ning and Huang, Panpan and Zhang, Peng and Wang, Qiancheng and Chen, Qinyu and Du, Qiushi and Ge, Ruiqi and Zhang, Ruisong and Pan, Ruizhe and Wang, Runji and Chen, R. J. and Jin, R. L. and Chen, Ruyi and Lu, Shanghao and Zhou, Shangyan and Chen, Shanhuang and Ye, Shengfeng and Wang, Shiyu and Yu, Shuiping and Zhou, Shunfeng and Pan, Shuting and Li, S. S. and Zhou, Shuang and Wu, Shaoqing and Yun, Tao and Pei, Tian and Sun, Tianyu and Wang, T. and Zeng, Wangding and Liu, Wen and Liang, Wenfeng and Gao, Wenjun and Yu, Wenqin and Zhang, Wentao and Xiao, W. L. and An, Wei and Liu, Xiaodong and Wang, Xiaohan and Chen, Xiaokang and Nie, Xiaotao and Cheng, Xin and Liu, Xin and Xie, Xin and Liu, Xingchao and Yang, Xinyu and Li, Xinyuan and Su, Xuecheng and Lin, Xuheng and Li, X. Q. and Jin, Xiangyue and Shen, Xiaojin and Chen, Xiaosha and Sun, Xiaowen and Wang, Xiaoxiang and Song, Xinnan and Zhou, Xinyi and Wang, Xianzu and Shan, Xinxia and Li, Y. K. and Wang, Y. Q. and Wei, Y. X. and Zhang, Yang and Xu, Yanhong and Li, Yao and Zhao, Yao and Sun, Yaofeng and Wang, Yaohui and Yu, Yi and Zhang, Yichao and Shi, Yifan and Xiong, Yiliang and He, Ying and Piao, Yishi and Wang, Yisong and Tan, Yixuan and Ma, Yiyang and Liu, Yiyuan and Guo, Yongqiang and Ou, Yuan and Wang, Yuduan and Gong, Yue and Zou, Yuheng and He, Yujia and Xiong, Yunfan and Luo, Yuxiang and You, Yuxiang and Liu, Yuxuan and Zhou, Yuyang and Zhu, Y. X. and Huang, Yanping and Li, Yaohui and Zheng, Yi and Zhu, Yuchen and Ma, Yunxian and Tang, Ying and Zha, Yukun and Yan, Yuting and Ren, Z. Z. and Ren, Zehui and Sha, Zhangli and Fu, Zhe and Xu, Zhean and Xie, Zhenda and Zhang, Zhengyan and Hao, Zhewen and Ma, Zhicheng and Yan, Zhigang and Wu, Zhiyu and Gu, Zihui and Zhu, Zijia and Liu, Zijun and Li, Zilin and Xie, Ziwei and Song, Ziyang and Pan, Zizheng and Huang, Zhen and Xu, Zhipeng and Zhang, Zhongyu and Zhang, Zhen},
   year={2025},
   pages={633–638} }

@inproceedings{bachmann2025judge,
  title={Judge decoding: Faster speculative sampling requires going beyond model alignment},
  author={Bachmann, Gregor and Anagnostidis, Sotiris and Pumarola, Albert and Georgopoulos, Markos and Sanakoyeu, Artsiom and Du, Yuming and Sch{\"o}nfeld, Edgar and Thabet, Ali and Kohler, Jonas},
  booktitle={International Conference on Learning Representations},
  volume={2025},
  pages={40894--40914},
  year={2025}
}

@misc{stern2018blockwise,
      title={Blockwise Parallel Decoding for Deep Autoregressive Models}, 
      author={Mitchell Stern and Noam Shazeer and Jakob Uszkoreit},
      year={2018},
      eprint={1811.03115},
      archivePrefix={arXiv},
      primaryClass={cs.LG},
      url={https://arxiv.org/abs/1811.03115}, 
}

@misc{miao2023specinfer,
      title={SpecInfer: Accelerating Generative Large Language Model Serving with Tree-based Speculative Inference and Verification}, 
      author={Xupeng Miao and Gabriele Oliaro and Zhihao Zhang and Xinhao Cheng and Zeyu Wang and Zhengxin Zhang and Rae Ying Yee Wong and Alan Zhu and Lijie Yang and Xiaoxiang Shi and Chunan Shi and Zhuoming Chen and Daiyaan Arfeen and Reyna Abhyankar and Zhihao Jia},
      year={2024},
      eprint={2305.09781},
      archivePrefix={arXiv},
      primaryClass={cs.CL},
      doi={https://doi.org/10.1145/3620666.3651335},
      url={https://arxiv.org/abs/2305.09781}, 
}

@misc{zhang2024draftverify,
      title={Draft \& Verify: Lossless Large Language Model Acceleration via Self-Speculative Decoding}, 
      author={Jun Zhang and Jue Wang and Huan Li and Lidan Shou and Ke Chen and Gang Chen and Sharad Mehrotra},
      year={2024},
      eprint={2309.08168},
      archivePrefix={arXiv},
      primaryClass={cs.CL},
      doi={https://doi.org/10.18653/v1/2024.acl-long.607},
      url={https://arxiv.org/abs/2309.08168}, 
}

@misc{qwenteam2026qwen35omnitechnicalreport,
      title={Qwen3.5-Omni Technical Report},
      author={{Qwen Team}},
      year={2026},
      eprint={2604.15804},
      archivePrefix={arXiv},
      primaryClass={cs.CL},
      url={https://arxiv.org/abs/2604.15804},
}

@misc{feng2026asd,
      title={Approximate Speculative Decoding}, 
      author={Yuannuo Feng and Zegang Peng and Yuxin Xie and Yubing Ye and Yizhe Chen and Wenshuai Yao and Wenyong Zhou and Wang Kang},
      year={2026},
      eprint={2608.03447},
      archivePrefix={arXiv},
      primaryClass={cs.LG},
      url={https://arxiv.org/abs/2608.03447}, 
}

@misc{hao2026,
      title={Cactus: Accelerating Auto-Regressive Decoding with Constrained Acceptance Speculative Sampling}, 
      author={Yongchang Hao and Lili Mou},
      year={2026},
      eprint={2604.04987},
      archivePrefix={arXiv},
      primaryClass={cs.LG},
      url={https://arxiv.org/abs/2604.04987}, 
}

@misc{urban2026adaptive,
      title={Margins, Not Windows: Training-Free Per-Step Lossy Speculative Decoding}, 
      author={Oszkár Urbán and Young D. Kwon and Stylianos I. Venieris and Cecilia Mascolo},
      year={2026},
      eprint={2609.02897},
      archivePrefix={arXiv},
      primaryClass={cs.CL},
      url={https://arxiv.org/abs/2609.02897}, 
}

@misc{song2026,
      title={MARS: Unleashing the Power of Speculative Decoding via Margin-Aware Verification}, 
      author={Jingwei Song and Xinyu Wang and Hanbin Wang and Xiaoxuan Lei and Bill Shi and Shixin Han and Eric Yang and Xiao-Wen Chang and Lynn Ai},
      year={2026},
      eprint={2601.15498},
      archivePrefix={arXiv},
      primaryClass={cs.LG},
      url={https://arxiv.org/abs/2601.15498}, 
}

@article{d3b61654b1204385b1a2e31f324d0229,
title = "Maximal coupling",
author = "Sheldon Goldstein",
year = "1979",
month = jan,
doi = "10.1007/BF00533259",
language = "English (US)",
volume = "46",
pages = "193--204",
journal = "Zeitschrift f{\"u}r Wahrscheinlichkeitstheorie und Verwandte Gebiete",
issn = "0044-3719",
publisher = "Springer Verlag",
number = "2",

}

@misc{gibbs2002choosingboundingprobabilitymetrics,
      title={On choosing and bounding probability metrics}, 
      author={Alison L. Gibbs and Francis Edward Su},
      year={2002},
      eprint={math/0209021},
      archivePrefix={arXiv},
      primaryClass={math.PR},
      url={https://arxiv.org/abs/math/0209021}, 
}

@misc{liu2026speculativedecodingperformanceillusion,
      title={Speculative Decoding: Performance or Illusion?}, 
      author={Xiaoxuan Liu and Jiaxiang Yu and Jongseok Park and Ion Stoica and Alvin Cheung},
      year={2026},
      eprint={2601.11580},
      archivePrefix={arXiv},
      primaryClass={cs.CL},
      url={https://arxiv.org/abs/2601.11580}, 
}

@misc{leviathan2023,
      title={Fast Inference from Transformers via Speculative Decoding}, 
      author={Yaniv Leviathan and Matan Kalman and Yossi Matias},
      year={2023},
      eprint={2211.17192},
      archivePrefix={arXiv},
      primaryClass={cs.LG},
      url={https://arxiv.org/abs/2211.17192}, 
}

@misc{chen2023accelerating,
      title={Accelerating Large Language Model Decoding with Speculative Sampling}, 
      author={Charlie Chen and Sebastian Borgeaud and Geoffrey Irving and Jean-Baptiste Lespiau and Laurent Sifre and John Jumper},
      year={2023},
      eprint={2302.01318},
      archivePrefix={arXiv},
      primaryClass={cs.CL},
      url={https://arxiv.org/abs/2302.01318}, 
}

@misc{cai2024,
      title={Medusa: Simple LLM Inference Acceleration Framework with Multiple Decoding Heads}, 
      author={Tianle Cai and Yuhong Li and Zhengyang Geng and Hongwu Peng and Jason D. Lee and Deming Chen and Tri Dao},
      year={2024},
      eprint={2401.10774},
      archivePrefix={arXiv},
      primaryClass={cs.LG},
      url={https://arxiv.org/abs/2401.10774}, 
}

@misc{holtzman2020,
      title={The Curious Case of Neural Text Degeneration}, 
      author={Ari Holtzman and Jan Buys and Li Du and Maxwell Forbes and Yejin Choi},
      year={2020},
      eprint={1904.09751},
      archivePrefix={arXiv},
      primaryClass={cs.CL},
      url={https://arxiv.org/abs/1904.09751}, 
}

@misc{li2025eagle3,
      title={EAGLE-3: Scaling up Inference Acceleration of Large Language Models via Training-Time Test}, 
      author={Yuhui Li and Fangyun Wei and Chao Zhang and Hongyang Zhang},
      year={2025},
      eprint={2503.01840},
      archivePrefix={arXiv},
      primaryClass={cs.CL},
      url={https://arxiv.org/abs/2503.01840}, 
}

@misc{deepseekv3,
      title={DeepSeek-V3 Technical Report}, 
      author={DeepSeek-AI and Aixin Liu and Bei Feng and Bing Xue and Bingxuan Wang and Bochao Wu and Chengda Lu and Chenggang Zhao and Chengqi Deng and Chenyu Zhang and Chong Ruan and Damai Dai and Daya Guo and Dejian Yang and Deli Chen and Dongjie Ji and Erhang Li and Fangyun Lin and Fucong Dai and Fuli Luo and Guangbo Hao and Guanting Chen and Guowei Li and H. Zhang and Han Bao and Hanwei Xu and Haocheng Wang and Haowei Zhang and Honghui Ding and Huajian Xin and Huazuo Gao and Hui Li and Hui Qu and J. L. Cai and Jian Liang and Jianzhong Guo and Jiaqi Ni and Jiashi Li and Jiawei Wang and Jin Chen and Jingchang Chen and Jingyang Yuan and Junjie Qiu and Junlong Li and Junxiao Song and Kai Dong and Kai Hu and Kaige Gao and Kang Guan and Kexin Huang and Kuai Yu and Lean Wang and Lecong Zhang and Lei Xu and Leyi Xia and Liang Zhao and Litong Wang and Liyue Zhang and Meng Li and Miaojun Wang and Mingchuan Zhang and Minghua Zhang and Minghui Tang and Mingming Li and Ning Tian and Panpan Huang and Peiyi Wang and Peng Zhang and Qiancheng Wang and Qihao Zhu and Qinyu Chen and Qiushi Du and R. J. Chen and R. L. Jin and Ruiqi Ge and Ruisong Zhang and Ruizhe Pan and Runji Wang and Runxin Xu and Ruoyu Zhang and Ruyi Chen and S. S. Li and Shanghao Lu and Shangyan Zhou and Shanhuang Chen and Shaoqing Wu and Shengfeng Ye and Shengfeng Ye and Shirong Ma and Shiyu Wang and Shuang Zhou and Shuiping Yu and Shunfeng Zhou and Shuting Pan and T. Wang and Tao Yun and Tian Pei and Tianyu Sun and W. L. Xiao and Wangding Zeng and Wanjia Zhao and Wei An and Wen Liu and Wenfeng Liang and Wenjun Gao and Wenqin Yu and Wentao Zhang and X. Q. Li and Xiangyue Jin and Xianzu Wang and Xiao Bi and Xiaodong Liu and Xiaohan Wang and Xiaojin Shen and Xiaokang Chen and Xiaokang Zhang and Xiaosha Chen and Xiaotao Nie and Xiaowen Sun and Xiaoxiang Wang and Xin Cheng and Xin Liu and Xin Xie and Xingchao Liu and Xingkai Yu and Xinnan Song and Xinxia Shan and Xinyi Zhou and Xinyu Yang and Xinyuan Li and Xuecheng Su and Xuheng Lin and Y. K. Li and Y. Q. Wang and Y. X. Wei and Y. X. Zhu and Yang Zhang and Yanhong Xu and Yanhong Xu and Yanping Huang and Yao Li and Yao Zhao and Yaofeng Sun and Yaohui Li and Yaohui Wang and Yi Yu and Yi Zheng and Yichao Zhang and Yifan Shi and Yiliang Xiong and Ying He and Ying Tang and Yishi Piao and Yisong Wang and Yixuan Tan and Yiyang Ma and Yiyuan Liu and Yongqiang Guo and Yu Wu and Yuan Ou and Yuchen Zhu and Yuduan Wang and Yue Gong and Yuheng Zou and Yujia He and Yukun Zha and Yunfan Xiong and Yunxian Ma and Yuting Yan and Yuxiang Luo and Yuxiang You and Yuxuan Liu and Yuyang Zhou and Z. F. Wu and Z. Z. Ren and Zehui Ren and Zhangli Sha and Zhe Fu and Zhean Xu and Zhen Huang and Zhen Zhang and Zhenda Xie and Zhengyan Zhang and Zhewen Hao and Zhibin Gou and Zhicheng Ma and Zhigang Yan and Zhihong Shao and Zhipeng Xu and Zhiyu Wu and Zhongyu Zhang and Zhuoshu Li and Zihui Gu and Zijia Zhu and Zijun Liu and Zilin Li and Ziwei Xie and Ziyang Song and Ziyi Gao and Zizheng Pan},
      year={2025},
      eprint={2412.19437},
      archivePrefix={arXiv},
      primaryClass={cs.CL},
      url={https://arxiv.org/abs/2412.19437}, 
}
\bibliographystyle{plainnat}
\endgroup
\clearpage
\appendix
\section{Theoretical Analysis and Proofs}
\label{app:proofs_candidate}
\label{sec:nsd_fidelity}

This appendix establishes the one-step distribution induced by nucleus-based
acceptance and then gives a conditional bound for sequences of emitted tokens.
The local identity is exact under the sampling assumptions below.

\subsection{Assumptions and Notation}
\label{app:nsd_assumptions}

Fix a finite vocabulary $\Vocab$ and an emitted prefix. Let $p$ be the
target sampling distribution and $q$ the actual proposal distribution after
any temperature or sampling transformations. The proposal is sampled as
$X\sim q$; using a different distribution in the acceptance ratio does not
satisfy this assumption. Let $S$ be a fixed set at this prefix, such as the
target nucleus in Equation~\eqref{eq:nsd_nucleus}, with ties resolved by a
fixed ordering. The proof holds for any such set, including the empty and
full vocabulary sets.

For $q(x)>0$, NSD accepts $x$ with probability
\begin{equation}
 a_S(x)=
 \begin{cases}
  1, & x\in S,\\
  \min\{1,p(x)/q(x)\}, & x\notin S.
 \end{cases}
 \label{eq:appendix_acceptance}
\end{equation}
Define
\begin{equation}
 \begin{aligned}
 d_+(x)&=[q(x)-p(x)]_+, & d_-(x)&=[p(x)-q(x)]_+,\\
 R&=\TV(q,p), & E_S&=\sum_{x\in S}d_+(x).
 \end{aligned}
 \label{eq:appendix_excess}
\end{equation}
Both $d_+$ and $d_-$ sum to $R$, and their supports are disjoint.
When $R>0$, every remaining rejection is followed by a draw from the
unchanged SD residual $\rho(x)=d_-(x)/R$, independent of the rejected
token's identity. The analysis concerns the resulting emitted token,
including this recovery draw. It does not condition on acceptance alone.

\subsection{Single-Step Output Distribution and Exact Error}

\begin{theorem}[Exact single-step NSD error and acceptance gain]
\label{thm:single_step_exact}
Under the assumptions above, for $R>0$ the emitted law is
\begin{equation}
 \widetilde p(x)=\min\{p(x),q(x)\}+\mathbf{1}_S(x)d_+(x)
                 +(R-E_S)\rho(x).
 \label{eq:output_candidate}
\end{equation}
It is normalized and satisfies
\begin{equation}
 \TV(\widetilde p,p)=E_S
 =\bar a_{\mathrm{NSD}}-\bar a_{\mathrm{SD}}\le R,
 \qquad
 \bar a_{\mathrm{NSD}}=1-R+E_S.
 \label{eq:appendix_local_identity}
\end{equation}
Here $\bar a$ is the probability of accepting the proposal. If $R=0$,
then $\widetilde p=p=q$ and the same identity holds with $E_S=0$.
\end{theorem}

\begin{proof}
The mass emitted by accepting a proposal is
\[
 q(x)a_S(x)=\min\{p(x),q(x)\}+\mathbf{1}_S(x)d_+(x).
\]
This mass identity also holds for $q(x)=0$, without evaluating $p(x)/q(x)$.
Summing over the vocabulary gives $1-R+E_S$, leaving rejection probability
$R-E_S$. Multiplying this probability by $\rho$ yields
Equation~\eqref{eq:output_candidate}. All terms are nonnegative because
$0\le E_S\le R$, and their sum is $(1-R)+E_S+(R-E_S)=1$.

Using $p(x)=\min\{p(x),q(x)\}+d_-(x)$, the signed deviation is
\[
 \widetilde p(x)-p(x)
 =\mathbf{1}_S(x)d_+(x)-\frac{E_S}{R}d_-(x).
\]
The positive and negative terms have disjoint support, so
\[
 \TV(\widetilde p,p)
 =\frac12\left(E_S+\frac{E_S}{R}\sum_xd_-(x)\right)=E_S.
\]
Standard SD accepts with probability $1-R$, giving the claimed acceptance
gain. If $R=0$, then $p=q$ and all proposals are accepted, so a residual
distribution is unnecessary. This also covers partial-support distributions.
\end{proof}

The result proves the local identity in Equation~\eqref{eq:nsd_local_identity}:
each unit of added acceptance mass introduces an equal amount of one-step
TV deviation. This concerns proposal acceptance, not the number of tokens
committed per verification or the end-to-end throughput gain.

\subsection{Exactness and Threshold Monotonicity}

\begin{corollary}[Exact single-step sampling]
\label{cor:single_step_lossless}
Under the same assumptions, $\widetilde p=p$ if and only if $E_S=0$,
equivalently $q(x)\le p(x)$ for every $x\in S$.
\end{corollary}
\begin{proof}
Theorem~\ref{thm:single_step_exact} gives $\TV(\widetilde p,p)=E_S$.
The finite sum defining $E_S$ vanishes exactly when each nonnegative
summand vanishes.
\end{proof}

Equality $p=q$ is sufficient but not necessary. If $S=\varnothing$,
the procedure reduces to exact SD. If $S=\Vocab$, then $E_S=R$ and
every proposal is accepted, giving $\widetilde p=q$.
For nested sets $S\subseteq S'$, $E_S\le E_{S'}$; both mean proposal
acceptance and one-step TV error are therefore nondecreasing as the nucleus
expands. These statements hold at fixed $p,q$ and prefix. They do not imply
monotonic accuracy or throughput across independently generated trajectories.

\subsection{Sequence-Level Coupling Bound}
\label{app:sequence_candidate}

Let $K_t(\cdot\mid h)$ be the decoder's true distribution for its next
\emph{emitted} token given history $h$, with internal state marginalized
conditional on that history. Let $p_t(\cdot\mid h)$ be the reference target
kernel. For each position, assume a uniform bound
\begin{equation}
 \sup_h\TV\bigl(K_t(\cdot\mid h),p_t(\cdot\mid h)\bigr)
 \le\epsilon_t,\qquad 0\le\epsilon_t\le1.
 \label{eq:appendix_kernel_bound}
\end{equation}
The supremum covers histories relevant to either process; conditional
kernels must be specified there. Write $P_K^{(T)}$ and $P_p^{(T)}$ for the
joint laws of the first $T$ emitted tokens. Here $T$ is an output horizon,
not the draft block size $K$. EOS can be made absorbing so that variable-length
responses are represented within a common finite horizon.

\begin{theorem}[Conditional sequence fidelity]
\label{thm:block_fidelity}
Under Equation~\eqref{eq:appendix_kernel_bound},
\begin{equation}
 \TV\bigl(P_K^{(T)},P_p^{(T)}\bigr)
 \le1-\prod_{t=1}^T(1-\epsilon_t)
 \le\sum_{t=1}^T\epsilon_t.
 \label{eq:sequence_candidate}
\end{equation}
There exists a coupling of the two sequences with complete agreement
probability at least $\prod_{t=1}^T(1-\epsilon_t)$.
\end{theorem}

\begin{proof}
Construct the sequences sequentially. Whenever their histories agree at
$h$, maximally couple $K_t(\cdot\mid h)$ and $p_t(\cdot\mid h)$.
Their next tokens then agree with probability at least $1-\epsilon_t$.
After histories diverge, use any coupling with the required conditional
marginals. If $M_t$ is the event of agreement through position $t$, then
\[
 \Pr(M_t)\ge(1-\epsilon_t)\Pr(M_{t-1}),\qquad \Pr(M_0)=1.
\]
Induction gives $\Pr(M_T)\ge\prod_t(1-\epsilon_t)$, including cases
where a factor is zero. The coupling inequality bounds TV by the probability
of disagreement. Finally, $1-\prod_t(1-\epsilon_t)\le\sum_t\epsilon_t$.
\end{proof}

\paragraph{Applying the bound to speculative blocks.}
Theorem~\ref{thm:single_step_exact} identifies the conditional error as
$E_{S_t(h)}$ only when $K_t$ is the complete one-proposal acceptance-and-recovery
kernel specified in Section~\ref{app:nsd_assumptions}. In a block decoder,
conditioning on earlier acceptances can change the law of pending proposals.
The relevant proposal law and verifier ratio must still match. Recovery
and bonus positions, cached proposals, and rollback state must also be included
in the next-output kernel; a recovery-only kernel is not generally equal to $p_t$.

A sufficient state-aware condition is that, for every reachable internal
state $z$, the next-output kernel $K_t(\cdot\mid h,z)$ has TV error at most
$\epsilon_t$ relative to the same $p_t(\cdot\mid h)$. Convexity then gives
\[
 \TV\left(\sum_z\mu_t(z\mid h)K_t(\cdot\mid h,z),p_t(\cdot\mid h)\right)
 \le\sum_z\mu_t(z\mid h)\TV\bigl(K_t(\cdot\mid h,z),p_t(\cdot\mid h)\bigr)
 \le\epsilon_t,
\]
where $\mu_t$ is the conditional state distribution. Exact target bonus
sampling has zero local error when it uses that same target conditional.
The sequence theorem is thus conditional on an output-kernel bound; it does
not establish such a bound for every implementation from token diagnostics alone.

\subsection{Uniform-Error Corollaries and Interpretation}

If $\epsilon_t\le\epsilon$ for all $t$, Theorem~\ref{thm:block_fidelity} gives
\begin{equation}
 \TV\bigl(P_K^{(T)},P_p^{(T)}\bigr)
 \le1-(1-\epsilon)^T\le\min\{1,T\epsilon\}.
\end{equation}
For $0<\delta<1$ and integer $T\ge1$, a sufficient condition for coupling
agreement at least $1-\delta$ is
\[
 \epsilon\le1-(1-\delta)^{1/T};
 \qquad\text{the simpler condition }\epsilon\le\delta/T\text{ also suffices}.
\]
For $0<\epsilon<1$, the product bound certifies this agreement whenever
$T\le\log(1-\delta)/\log(1-\epsilon)$. This is a sufficient bound on the
output horizon, not a validated prescription for choosing a speculative block size.

\begin{corollary}[Exact sequence sampling]
\label{cor:sequence_lossless}
If $K_t(\cdot\mid h)=p_t(\cdot\mid h)$ at every relevant history and position,
then $P_K^{(T)}=P_p^{(T)}$ for every finite horizon $T$. In particular, this
holds for a sequence of complete one-step NSD kernels satisfying the stated
proposal and recovery assumptions with $E_{S_t(h)}=0$ everywhere.
\end{corollary}
\begin{proof}
Set every $\epsilon_t=0$ in Theorem~\ref{thm:block_fidelity}.
\end{proof}

The agreement statement concerns an explicitly constructed coupling, not
independently sampled sequences or runs sharing an arbitrary random seed.
It is not a statistical confidence level or a guarantee of unchanged task
accuracy. For a fixed correctness event $A$, a valid sequence-TV bound does
imply $|P_K^{(T)}(A)-P_p^{(T)}(A)|\le\TV(P_K^{(T)},P_p^{(T)})$; however,
using that inequality requires a justified sequence-level bound.
Diagnostic averages over observed nucleus-accepted tokens do not supply
the uniform bounds in Equation~\eqref{eq:appendix_kernel_bound}.

\section{Method conventions and implementation correspondence}
\label{app:method_candidate}

\paragraph{Cumulative-mass boundaries.}
The shared gate sums transformed target mass over tokens with strictly greater target logits.
Under an unrestricted positive-temperature softmax, this ordering agrees with Equation~\ref{eq:nsd_nucleus}.
All equally scored tokens share membership, so the result need not be the smallest sorted set reaching the threshold. Ordinary top-p sampling is a separate transformation that truncates and renormalizes a distribution. NSD's nucleus threshold does not truncate replacement or bonus sampling.
The MTP adapter accepts thresholds only in $(0,1]$. The zero-threshold examples describe the mathematical/shared membership convention.

\paragraph{DSpark.}
The proposal routine returns sampled tokens and associated sampling logits, including the optional Markov-head path.
The backbone uses a noncausal draft pass, while the target evaluates the proposed prefix causally.
An optional confidence gate can select a prefix after sampling.
When interpreting $q_i$ as an actual proposal law, such selection must be disabled or its conditioning accounted for.
The historical confidence and checkpoint settings are not recoverable from the current source alone.

\paragraph{Nemotron linear self-speculation.}
The sampled path uses a bidirectional draft pass with LoRA enabled and a causal target pass with adapters disabled.
It samples proposal positions from the draft probabilities and treats the target seed separately.
The cache is cropped to the seed plus accepted drafts. Output wrappers trim EOS and length-budget overshoot.
The reported SD configuration can invoke an external official generation routine, so equivalence to the local sampled path with nucleus-based acceptance disabled requires that external source and historical configuration.

\paragraph{Qwen3.5-MTP.}
The adapter invokes the vLLM rejection kernel one proposal at a time and retains its replacement output when the nucleus-based condition is not satisfied.
Its gate takes a softmax of the target-logit row after the adapter's processors. The standard sampler may apply further transformations.
The reported temperature-one, unrestricted policy is the correspondence considered in the main text.
The adapter supports one active sequence and delegates proposal generation, bonus sampling, stopping, and caches to vLLM.
The historical external kernel is unavailable locally. Adapter inspection does not certify those branches.

\paragraph{Numerical and state conventions.}
The shared ratio uses $\max(q_i(x_i),10^{-8})$ in the denominator, and residual mass at most $10^{-8}$ triggers sampling from $p_i$.
The exact acceptance--deviation identity assumes neither modification and requires residual recovery independent of the rejected token's identity.
Symbolic trace checks verify the example's first rejection, full acceptance, EOS, and budget behavior. They do not execute backend caches or GPU kernels.



\section{Mechanism comparison with closely related work}
\label{app:mechanism_comparison}
The comparison below separates the rule for retaining a proposal from the rule
for generating the next token after a failure. This distinction matters because
an acceptance criterion alone does not determine an output distribution.
Entries describe the cited formulations, not interchangeable implementations
or a controlled performance comparison.

\begin{table}[htbp]
\centering\small
\renewcommand{\arraystretch}{1.16}
\begin{tabular}{@{}>{\raggedright\arraybackslash}p{0.17\linewidth}>{\raggedright\arraybackslash}p{0.37\linewidth}>{\raggedright\arraybackslash}p{0.39\linewidth}@{}}
\toprule
Method & Selection or intervention & Recovery and scope \\
\midrule
SD & Stochastic ratio acceptance $\min(1,p/q)$. & Sample from normalized $[p-q]_+$ at rejection; preserves the chosen target law under the usual assumptions. \\
NSD & Retain ratio acceptance and accept in-nucleus failures. & Keep the same residual; local added acceptance and TV error equal $E_S$. No request-level budget. \\
ASD \citep{feng2026asd} & Greedy exceptions constrained by logit regret, a block cap, and a persistent request budget. & Target argmax recovery or bonus under the committed prefix. Zero budget restores greedy verification; regret accounting is not an NSD TV bound. \\
MARS \citep{song2026} & Accept the target top-1 token, or its top-2 token when the raw-logit ratio $z_{(2)}/z_{(1)}>\theta$. & On failure, emit the target top-1 token and discard the remaining draft. The stated rule relaxes token matching; it does not use the stochastic $[p-q]_+$ residual. \\
Cactus \citep{hao2026} & Select an adjusted distribution under a divergence constraint. & Uses generalized acceptance/recovery equations. NSD instead fixes the original residual and analyzes the induced deviation. \\
\bottomrule
\end{tabular}
\caption{Mechanism distinctions. Here $p$ and $q$ always denote target and actual
proposal laws; notation used in other papers is not imported into these columns.}
\label{tab:mechanism_comparison}
\end{table}

\FloatBarrier
\section{Additional Proposal Mechanisms}
\label{app:additional_proposals}

We compare SD with NSD at an
acceptance threshold of $0.95$ using DFlash-b16 and EAGLE-3-TTT7 with Qwen3-4B as the target.
Both comparisons use temperature $1.0$, target sampling top-$p=1.0$,
and five seeds.

\begin{table}[htbp]
\centering
\small
\setlength{\tabcolsep}{5pt}
\caption{Task performance and mean accepted length for additional proposal mechanisms.
Task performance is mathematical accuracy or functional pass@1 (\%), reported as
the mean over five seeds. $\tau$ includes the
resample or bonus token.}
\label{tab:additional_proposal_quality}
\begin{tabular}{llrrrr}
\toprule
& & \multicolumn{2}{c}{Acc. (\%)} & \multicolumn{2}{c}{$\tau$} \\
\cmidrule(lr){3-4}\cmidrule(lr){5-6}
Proposal & Benchmark & SD & NSD & SD & NSD \\
\midrule
DFlash-b16 & MATH500 & 83.88 & 84.28 & 5.75 & 6.32 \\
& AIME25 & 17.33 & 18.67 & 4.61 & 5.13 \\
& HumanEval & 84.51 & 83.17 & 5.48 & 6.20 \\
& MBPP & 62.41 & 63.50 & 4.77 & 5.29 \\
\midrule
EAGLE-3-TTT7 & MATH500 & 83.64 & 83.32 & 4.54 & 5.32 \\
& AIME25 & 20.00 & 20.00 & 3.83 & 4.78 \\
& HumanEval & 84.39 & 83.90 & 4.16 & 4.72 \\
& MBPP & 62.65 & 61.63 & 3.70 & 4.16 \\
\bottomrule
\end{tabular}
\end{table}

\begin{table}[htbp]
\centering
\small
\caption{Output throughput (tokens/s) under additional proposal mechanisms.
DFlash reports per-benchmark results for seed 30.
Both methods ran unusually slowly on this seed; these are observed
throughputs, not measurements under verified uncontended conditions.
EAGLE-3 reports per-benchmark throughput for uncontended seeds 3, 7, and 22.}
\label{tab:additional_proposal_throughput}
\begin{tabular}{llrr}
\toprule
Proposal & Benchmark & SD & NSD \\
\midrule
DFlash-b16 & MATH500 & 115.92 & 140.77 \\
& AIME25 & 87.48 & 115.15 \\
& HumanEval & 118.54 & 141.18 \\
& MBPP & 102.36 & 121.65 \\
\midrule
EAGLE-3-TTT7 & MATH500 & 157.25 & 188.61 \\
& AIME25 & 115.63 & 155.62 \\
& HumanEval & 153.56 & 184.06 \\
& MBPP & 139.97 & 164.93 \\
\bottomrule
\end{tabular}
\end{table}

\subsection{MTP with Two Proposal Tokens}
\label{app:mtp_block2}

We additionally report the Qwen3.5-4B MTP comparison \citep{qwenteam2026qwen35omnitechnicalreport} with two
proposal tokens per block using the vLLM backend. Temperature is $1.0$, target sampling top-$p$ is $1.0$, and the NSD
acceptance threshold is $0.95$.

\begin{table}[htbp]
\centering
\small
\caption{Qwen3.5-4B MTP results with block size $K=2$.
Acc. is mean accuracy or pass@1 (\%);
$\tau$ is weighted by verification-block counts.
The maximum committed length is three, including the continuation token.}
\label{tab:mtp_block2}
\begin{tabular}{lrrrr}
\toprule
& \multicolumn{2}{c}{Acc. (\%)} & \multicolumn{2}{c}{$\tau$} \\
\cmidrule(lr){2-3}\cmidrule(lr){4-5}
Benchmark & SD & NSD & SD & NSD \\
\midrule
MATH500 & 92.73 & 92.33 & 2.63 & 2.95 \\
AIME25 & 35.56 & 34.44 & 2.57 & 2.94 \\
HumanEval & 78.25 & 80.89 & 2.55 & 2.95 \\
MBPP & 55.12 & 57.85 & 2.42 & 2.92 \\
\bottomrule
\end{tabular}
\end{table}

The throughput for NSD is 179.90 tokens/s, which is $1.70\times$ the SD throughput of 105.90 tokens/s.

\FloatBarrier
\section{Limitations}
\label{sec:limitations}

Although our experiments demonstrate the effectiveness of Nucleus Speculative Decoding (NSD) across multiple settings, one limitation of the current study is the range of model sizes evaluated. Because of hardware constraints and the availability of Nemotron and DSpark checkpoints, our evaluation focuses primarily on relatively compact target models and corresponding draft models.
\end{document}